\documentclass{article}
\usepackage[nohyperref, accepted]{icml2022}

\usepackage[utf8x]{inputenc}

\usepackage{algorithm}
\usepackage{algorithmic}
\usepackage[colorlinks=true, linkcolor=red, urlcolor=cyan, citecolor=cyan]{hyperref}
\usepackage{url}
\usepackage{booktabs}
\usepackage{color}
\usepackage{placeins}
\usepackage{amsmath}
\usepackage{amsthm}
\usepackage{amssymb}
\usepackage{caption}
\usepackage{subcaption}
\usepackage{tikz}
\usepackage{comment}
\usepackage{bm}
\usepackage{array}
\usepackage{scalefnt}
\usepackage{mathtools}
\usepackage{footnote}
\usepackage{amsmath}
\usepackage{amssymb}
\usepackage{graphicx}
\usepackage[colorinlistoftodos]{todonotes}
\usepackage{thm-restate}
\usepackage{stmaryrd} 
\usepackage{dsfont}
\usepackage{array}

\newcommand{\finterval}{{\Gamma_{\mathrm{interval}}}}

\newcommand{\confunc}{\phi}

\newcommand{\metricfont}{\texttt}

\usepackage{enumitem}

\newcommand{\Fc}{{\mathcal{F}}}

\newcommand{\Qc}{{\mathcal{Q}}}
\newcommand{\Rc}{{\mathcal{R}}}

\newcommand{\Xc}{{\mathcal{X}}}
\newcommand{\Yc}{{\mathcal{Y}}}

\newcommand{\Eb}{{\mathbb{E}}}

\newcommand{\Ebb}{{\mathbb{E}}}
\newcommand{\Fbb}{{\mathbb{F}}}

\newcommand{\Rbb}{{\mathbb{R}}}

\newtheorem{prop}{Proposition}

\newtheorem{definition}{Definition}

\newcommand\numberthis{\addtocounter{equation}{1}\tag{\theequation}}

\newcommand{\calfunc}{\varphi}
\newcommand{\calscore}{S}
\newcommand{\intalg}{\psi}
\newcommand{\intfunc}{q}
\newcommand{\intfuncset}{\Qc}

\newtheorem{example}{Example}
\newtheorem{observation}{Observation}

\icmltitlerunning{Modular Conformal Calibration}

\begin{document}

\twocolumn[

\icmltitle{Modular Conformal Calibration}

\icmlsetsymbol{equal}{*}

\begin{icmlauthorlist}
\icmlauthor{Charles Marx}{equal,stanford}
\icmlauthor{Shengjia Zhao}{equal,stanford}
\icmlauthor{Willie Neiswanger}{stanford}
\icmlauthor{Stefano Ermon}{stanford}
\end{icmlauthorlist}

\icmlaffiliation{stanford}{Computer Science Department, Stanford University}

\icmlcorrespondingauthor{Charles Marx}{ctmarx@stanford.edu}

\icmlkeywords{Machine Learning, ICML}

\vskip 0.3in
]

\printAffiliationsAndNotice{\icmlEqualContribution} %

\begin{abstract}
Uncertainty estimates must be calibrated (i.e., accurate) and sharp (i.e., informative) in order to be useful. This has motivated a variety of methods for {\em recalibration}, which use held-out data to turn an uncalibrated model into a calibrated model. However, the applicability of existing methods is limited due to their assumption that the original model is also a probabilistic model.
We introduce a versatile class of algorithms for recalibration in regression that we call \emph{modular conformal calibration} (MCC). This framework allows one to transform any regression model into a calibrated probabilistic model. The modular design of MCC allows us to make simple adjustments to existing algorithms that enable well-behaved distribution predictions. We also provide finite-sample calibration guarantees for MCC algorithms. Our framework recovers isotonic recalibration, conformal calibration, and conformal interval prediction, implying that our theoretical results apply to those methods as well.
Finally, we conduct an empirical study of MCC on 17 regression datasets. Our results show that new algorithms designed in our framework achieve near-perfect calibration and improve sharpness relative to existing methods.
\end{abstract}

\section{Introduction} 
\renewcommand{\comment}[1]{}

Uncertainty estimates can inform human decisions \citep{pratt1995introduction, berger2013statistical}, flag when an automated decision system requires human review \citep{kang2021statistical}, and serve as an internal component of automated systems. For example, uncertainty informs treatment decisions in medicine \citep{begoli2019need} and supports safety in autonomous navigation \citep{michelmore2018evaluating}. 
In such settings, the benefits of accounting for uncertainty hinge on our ability to produce \emph{calibrated} uncertainty estimates---e.g., of those events to which one assigns a probability of 90\%, the events should indeed occur 90\% of the time. A model that is not calibrated can consistently make confident predictions that are incorrect. 

Many models, such as neural networks \citep{guo2017calibration} and Gaussian processes \citep{rasmussen2003gaussian, tran2019calibrating}, achieve high accuracy but have poorly calibrated or absent uncertainty estimates. In other cases, a pretrained model is released for wide use and it is difficult to guarantee that it will produce calibrated uncertainty estimates in new settings \citep{zhao2021calibrating}. This leads us to the question: \textit{how can we safely deploy models with high predictive value but poor or absent uncertainty estimates?}

These challenges have motivated work on \emph{recalibration}, whereby a model with poor uncertainty estimates is transformed into a probabilistic model that outputs calibrated probabilities \citep{kuleshov2018accurate, vovk2020conformal, niculescu2005predicting, chung2021uncertainty}. Recalibration methods are attractive because they require only black-box access to a given model and can return well-calibrated probabilistic predictions. 

However, calibration is not the only goal of probabilistic models. 
It is also important for a probabilistic model to predict sharp (i.e., low variance) distributions to convey more information. Furthermore, recalibration methods need to be data efficient to calibrate models in data poor regimes.

In this paper, we introduce \emph{modular conformal calibration} (MCC), a class of algorithms that unifies existing recalibration methods and gives well-behaved distribution predictions from any model. Our main contributions are:

\begin{enumerate}[parsep=0pt, topsep=-3pt, itemsep=5pt, leftmargin=5mm]
    \item We introduce modular conformal calibration, a class of algorithms for recalibration in regression, which can be applied to recalibrate almost any regression model. MCC unifies isotonic calibration \citep{kuleshov2018accurate}, conformal calibration \citep{vovk2020conformal}, and conformal interval prediction \citep{vovk2005algorithmic} under a single theoretical framework, and additionally leads to new algorithms.
    \item We provide finite-sample calibration guarantees, showing that MCC can achieve $\epsilon$ calibration error with $O(1/\epsilon)$ samples. These results also apply to the aforementioned recalibration methods that MCC unifies. 
    \item We conduct an empirical study on 17 datasets to compare the performance of recalibration methods in practice. We find that new algorithms within our framework outperform existing methods in terms of both sharpness and proper scoring rules.
\end{enumerate}

\section{Background}

\begin{figure*}
    \centering
    \includegraphics[width=0.9\linewidth]{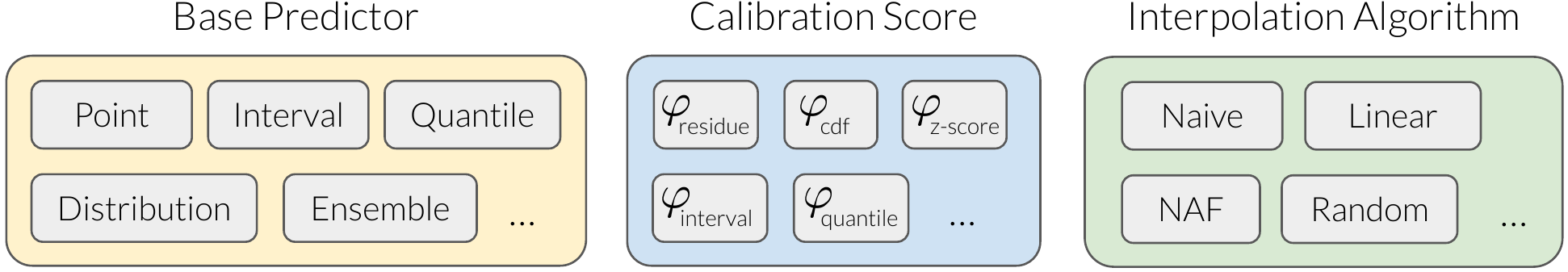}
    \caption{A combination of the three components defines a modular conformal calibration algorithm. %
    }
    \label{fig:design_choices}
\end{figure*}

\newcommand{\iid}{i.i.d.\,}
\newcommand{\indic}{\mathds{1}}
\newcommand{\indicarg}[1]{\mathds{1} \left\{ #1 \right\}}
\newcommand{\uniform}{\mathrm{Uniform}}
\newcommand{\normal}{\mathrm{Normal}}
\newcommand{\truedist}{\Fbb}

Given an input feature vector $x \in \Xc$ (e.g., a satellite image), we want to predict a label $y \in \Yc$ (e.g., the temperature tomorrow). We consider regression problems where $\Yc=\Rbb$. We assume there is a true distribution $\truedist_{XY}$ over $\Xc \times \Yc$, and we have access to $n$ \iid examples $(X_i, Y_i) \sim \truedist_{XY}$. 

Given a feature vector $x$, our goal is to predict the conditional distribution of $Y$ given $X=x$, denoted $\truedist_{Y\mid x}$. A \emph{distribution predictor} is a function $H: \Xc \to \Fc(\Yc)$ that takes a feature vector $x$ as input and returns $H[x]$, a  cumulative distribution function (CDF) over $\Yc$. Note that $H[x] \in \Fc(\Yc)$ is a function, intended to approximate $\truedist_{Y\mid x}$ by mapping any input $y \in \Yc$ to a value $H[x](y)$ in $[0, 1]$. 

We consider a two-stage process to learn a distribution predictor from data. 
In the first stage, we train the \emph{base predictor} $f: \Xc \to \Rc$, which maps a feature vector $x$ to a prediction in a space $\Rc$. The base predictor can be any model (e.g., neural network, support vector machine) and the prediction can be of any type; for example, $f$ could give a point prediction for the mean of $\truedist_{Y \mid x}$ (i.e., $\Rc=\Rbb$) or $f$ could give an interval prediction that is likely to contain $Y$ (i.e., $\Rc=\Rbb^2$). Alternatively, $f$ could predict a Gaussian distribution that approximates the distribution of the label (i.e., $\Rc=\Fc(\Yc)$). 

Regardless of the base predictor we choose, the second step is to \emph{recalibrate} the base predictor: we translate the base predictor $f$ into a calibrated distribution predictor $H$. We construct $H$ by fitting a wrapper function around $f$, meaning that $H[x]$ only depends on $x$ via the prediction $f(x)$. 

We focus on the second stage, that of recalibration. 
For those interested in which base predictors yield the best recalibrated predictors, see Section \ref{sec:experiment} for an empirical study.

\begin{example}[Linear Regression]
\label{ex::prediction}
\emph{Consider a linear regression problem where $Y_i=\beta^\top X_i + \epsilon_i$ for $X_i\in \Rbb^d$, $\beta \in \Rbb^d$, and $\epsilon_i \in \Rbb$ distributed \iid with known CDF $\truedist_\epsilon$. Imagine we are given a base predictor $f(x)=\beta^T x$ that perfectly predicts the mean of $\truedist_{Y \mid x}$. Then we can construct a perfect distribution predictor $H[x]: \Yc \rightarrow [0, 1]$ by defining:
\begin{align*}
    H[x](y) := \truedist_\epsilon(y - f(x)) = \truedist_{Y \mid x}(y)
\end{align*}
Note that the CDF prediction $H[x]$ only depends on $x$ through the point prediction $f(x)$, but still gives perfect distribution predictions. 
}\end{example}

\paragraph{Calibration} Optimally, the distribution predictor $H$ will output for each value $x$ the true conditional CDF, $H[x]=\truedist_{Y\mid x}$, as in Example \ref{ex::prediction}. However, many feature vectors $x$ only appear once in our data, making it impossible to learn a perfect distribution predictor $H$ from data without additional assumptions (such as the assumptions of linearity and \iid noise in Example \ref{ex::prediction}). Instead of making additional assumptions, we instead aim for \emph{calibration}, a weaker property than perfect distribution prediction that can be obtained in practice. 

Recall that for any random variable $Y$, the probability integral transform $\truedist_Y(Y)$ obtained by evaluating the CDF with random input $Y$ follows a standard uniform distribution. We should expect the same behavior from our predicted CDFs; we should observe that $H[X](Y)$ also follows a standard uniform distribution. For example, the observed label should be greater than the predicted 95th percentile for approximately 5\% of examples.

\begin{definition}
Given a distribution predictor $H: \Xc \to \Fc(\Yc)$, we say that $H$ is \emph{calibrated} if $H[X](Y)$ follows a standard uniform distribution. Formally, $H$ is calibrated if
\begin{equation}
    \Pr(H[X](Y) \leq p) = p, \quad \text{for all } p \in [0, 1]
    \label{eq::calibration}
\end{equation}
Similarly, for a value $\epsilon \geq 0$, we say that $H$ is \emph{$\epsilon$-calibrated} if Equation \eqref{eq::calibration} is only violated by at most $\epsilon$:
\begin{equation}
    \Pr(H[X](Y) \leq p) \in p \pm \epsilon, \quad \textrm{for all } p \in [0, 1]
    \label{eq::epscalibration}
\end{equation}
 \end{definition}

Calibration is a necessary but not sufficient condition for making good distribution predictions. 
Note that a distribution predictor $H[x]=\truedist_Y$ that ignores $x$ and returns the marginal cdf for $Y$ will be calibrated, but not useful. Thus, distribution predictions should be as \emph{sharp} (i.e., highly concentrated) as possible, conditioned on being calibrated.

\paragraph{Related Work} Post-hoc uncertainty quantification is an active field of research. 
Platt scaling \cite{platt1999probabilistic} and isotonic regression \cite{niculescu2005predicting} are popular methods for recalibrating binary classifiers. Platt scaling fits a logistic regression model to the scores given by a model, and isotonic regression learns a nondecreasing map from scores to the unit interval. Quantile regression \citep[e.g.,]{romano2019conformalized, chung2020beyond} simultaneously estimates multiple quantiles of the label distribution, often via the pinball loss, which can then be combined to construct calibrated distribution predictions.

Isotonic calibration \citep{kuleshov2018accurate} is an effective strategy for recalibrating a base predictor that already makes distribution predictions. Isotonic calibration computes the empirical quantiles (i.e., $f[X_i](Y_i)$) of a distribution predictor on the calibration dataset, and uses isotonic regression to adjust the empirical quantiles so that they are uniform on $[0, 1]$. 
Conformal calibration \citep{vovk2020conformal} is similar to isotonic calibration, except it uses a randomized function to adjust the empirical quantiles instead of isotonic regression. This yields strong calibration guarantees, at the cost of discontinuous and randomized distribution predictions.

Conformal prediction \citep{vovk2005algorithmic} is a general approach to uncertainty quantification that produces prediction sets (i.e., interval predictions) with guaranteed coverage, instead of distribution predictions.
In the context of these prediction sets, some prior work has also studied the connection between conformal prediction and calibration \citep{lei2018distribution, gupta2020distribution, angelopoulos2021gentle}.
Our work builds on isotonic calibration, conformal calibration, and conformal prediction to construct novel recalibration algorithms for arbitrary base predictors.

\section{Modular Conformal Calibration}
\label{sec:modular}

In this section, we introduce a new class of recalibration procedures and provide calibration guarantees for this class of algorithms. We begin with a simple example in which we recalibrate a point predictor, then we generalize this reasoning to introduce modular conformal calibration. We conclude this section by enumerating design choices our framework introduces.

\subsection{Warm-up} 
\label{sec:simple_example}

We start with a simple example to introduce the main idea. In this example, we turn a point predictor into a calibrated distribution predictor. 

\begin{enumerate}[parsep=0pt, topsep=-3pt, itemsep=5pt, leftmargin=5mm]

\item Suppose we have a base predictor $f: \Xc \to \Rbb$ that uses a satellite image $X$ to produce a point estimate for the temperature the following day $Y$. Additionally, we are given a dataset $(X_1, Y_1), \dots, (X_n, Y_n)$ where we observe both the satellite image and temperature. Now, given a new satellite image $X_*$, we want to predict the (unobserved) temperature $Y_*$. It is important for us to quantify our uncertainty if this prediction will inform a decision---we will likely behave differently if the temperature is certain to be within 2 degrees of our estimate, versus if it could differ by 20 degrees.

\item We can apply the residue function $\calfunc(f(x), y)=y - f(x)$ to our predictions, giving $\calscore_1 = Y_1 - f(X_1), \cdots, \calscore_n = Y_n - f(X_n)$; if we knew the label $Y_*$, we could also apply the residue score to our test example to compute a residue $S_*=Y_* - f(X_*)$. If the data is i.i.d. then the residues are also i.i.d. random variables. %

\item We can consider how large or small $S_*$ is among the set of residues $\lbrace \calscore_1, \cdots, \calscore_n \rbrace$. Intuitively, because of the i.i.d. assumptions, $S_*$ is equally likely to be the smallest, 2nd smallest, ..., largest element. %
Formally, if we define the ranking function $q$ as:
\begin{align*}
    q(t) := \frac{1}{n} \sum_{i=1}^n \indic\lbrace t \leq \calscore_i \rbrace\
\end{align*}
then up to discretization error 
\begin{align*}
    \Pr[ q(S_*) \leq c] \approx c, \forall c \in [0, 1] \numberthis\label{eq:calibration_eq} 
\end{align*}

\end{enumerate}
In fact, Eq.(\ref{eq:calibration_eq}) is exactly the definition of calibration so the reasoning above proves that the CDF predictor $H[x](y)=q(\calfunc(f(x), y))$ is approximately calibrated. We also have to show that $H[x]$ is a CDF. This is easy to prove, as $q(\calfunc(f(x), y))$ is a nondecreasing function in $y$. We conclude that $H$ is an approximately calibrated CDF predictor.

\subsection{Components of a Recalibration Algorithm}

\begin{algorithm}[tb]
\caption{Modular Conformal Calibration} 
  \label{alg:general}
\begin{algorithmic}
\STATE {\bfseries Input:} base predictor $f: \Xc \to \Rc$, calibration score $\calfunc: \Rc \times \Yc \to \Rbb$ and interpolation algorithm $\intalg$
\STATE {\bfseries Input:} calibration dataset $(X_1, Y_1), \cdots, (X_n, Y_n)$ %
\STATE Compute calibration scores $\calscore_i = \calfunc(f(X_i),Y_i)$ for $i=1, \dots, n$
\STATE Run the interpolation algorithm $\intfunc = \intalg\left(\calscore_{1}, \cdots, \calscore_{n}\right) $
\STATE {\bfseries Return:} the CDF predictor $H[x](y) = \intfunc(\calfunc(f(x), y)$
\end{algorithmic}
\end{algorithm}

We organize the design choices within modular conformal calibration into three decisions: 

\begin{enumerate}[parsep=0pt, topsep=-3pt, itemsep=5pt, leftmargin=5mm]
    \item \textbf{Base predictor.} In the first step, we choose a base predictor. This can be any prediction function $f: \Xc \to \Rc$. There are no restrictions on the prediction space $\Rc$, as long as we can define a compatible calibration score. The only requirement is that $f$ is not learned on the calibration dataset $(X_1,Y_1), \cdots, (X_n, Y_n)$, but it could be learned on any different dataset. In the previous example, the base predictor is a point prediction function. 
    \item \textbf{Calibration score.} In the second step, we choose a calibration score, which is any function $\calfunc: \Rc \times \Yc \to \Rbb$ that is monotonically strictly increasing in $y$. In the previous example, the calibration score is the residue $\calfunc(f(x), y) = y - f(x)$. Intuitively, the calibration score should reflect how large $y$ is relative to our prediction $f(x)$. We can then compute the calibration score for each sample in the training set $\calscore_i = \calfunc(X_i, Y_i); i=1, \cdots, n$. For convenience of computing rankings, we sort the scores into $\calscore_{(1)} \leq \calscore_{(2)} \leq \cdots \leq \calscore_{(n)}$. 
    \item \textbf{Interpolation algorithm.}  Finally we need a map from the calibration score to the final CDF output. In example 1 we constructed an interpolation map (the function $q$) by mapping any score in $(\calscore_{(i-1)}, \calscore_{(i)}]$ to $i/n$. The interpolation algorithm we use here is a very simple step function. However, the resulting CDFs are not continuous which may be inconvenient (e.g., if we want to compute the log likelihood). 
\end{enumerate}

More generally we can use any interpolation algorithm: let $\intfuncset$ be the set of monotonically non-decreasing functions $\Rbb \to [0, 1]$. An interpolation algorithm is a map $\intalg: \Rbb^n  \to \intfuncset$. An interpolation algorithm maps the calibration scores $S_{1}, \dots, S_{n}$ to a function $\intfunc$ such that $\intfunc(S_{1}), \dots, \intfunc(S_{n})$ are approximately evenly spaced on the unit interval. 

\begin{definition}
\label{def:accurate}
An interpolation function $\intalg: \Rbb^n \to \intfuncset$ is \emph{$\lambda$-accurate} if for any distinct inputs $(u_1, u_2, \dots, u_n) \in \Rbb^n$ the function $\intfunc = \intalg(u_1, u_2, \cdots, u_n)$ maps the $i$-th smallest input $u_{(i)}$ close to $i/(n + 1)$:
\begin{equation}
\intfunc\left(u_{(i)}\right) \in  \frac{i \pm \lambda}{n+1}, \quad \text{for all }i = 1, \dots, n
\end{equation}
 If $\intalg$ is a randomized function, then the statement is quantified by almost surely. 
\end{definition}

If a $\lambda$-accurate interpolation algorithm is applied to calibration scores computed on a held-out dataset, the function $\intfunc \circ \calfunc$ will be approximately calibrated on that dataset. 
We can write the full process for making a CDF prediction as:
\begin{equation}
H[x](y) = \underbrace{\intfunc(\calfunc(f(x), y))}_{\text{prediction for} \Pr(Y \leq y \mid X = x)}
\end{equation}
This three step process of applying a base predictor $f$, calibration score $\calfunc$, and interpolation function $\intfunc$ (learned by an interpolation algorithm $\intalg$) is detailed in Algorithm~\ref{alg:general} and illustrated in Figure~\ref{fig:design_choices}. Now, we formalize the intuition that $H$ will be calibrated into a formal guarantee.

\begin{restatable}{theorem}{thmcalibration}
\label{thm:calibration}
For any base predictor $f$, calibration score $\calfunc$, and $\lambda$-accurate interpolation algorithm $\intalg$ such that the random variable $\calfunc(f(X), Y)$ is absolutely continuous, Algorithm~\ref{alg:general} is $\frac{1+\lambda}{n+1}$-calibrated. %
\end{restatable}

See Appendix \ref{app:proofs} for a proof of Theorem \ref{thm:calibration}. Similar to conformal interval prediction, there is a rather mild regularity assumption: $\calfunc(f(X), Y)$ has to be absolutely continuous, i.e. two i.i.d. samples $(X_1, Y_1)$ and $(X_2, Y_2)$ almost never have the same 
score $\calfunc(f(X_1), Y_1) \neq \calfunc(f(X_2), Y_2)$. In our warm-up example in Section \ref{sec:simple_example}, this condition requires that two samples $(X_1, Y_1)$, $(X_2, Y_2)$ almost never have exactly the same residue $Y_1 - f(X_1)  \neq Y_2 - f(X_2)$.

\section{Choosing a Recalibration Algorithm}

In this section, we describe natural choices for the calibration score and interpolation algorithm, given different base predictors. A main motivation for introducing the modular conformal calibration framework is to make it easy to develop new recalibration procedures.
Any pairing of the calibration scores and interpolation algorithms described in this section results in a recalibration algorithm with the finite-sample calibration guarantee given by Theorem \ref{thm:calibration}.

\subsection{The Base Predictor}

In some cases the base predictor will be fixed, such as when fine-tuning a pretrained model to be calibrated in a new setting. In other cases, we have end-to-end control of the  training process. In these cases, we must answer the question: \emph{Which base predictor should I train to get the best calibrated distribution predictor?}

An obvious choice is to learn a distribution predictor as the base predictor then recalibrate if needed. However, there is no guarantee that this will produce better results than learning a different type of base predictor (e.g., one of the prediction types in Table \ref{tab:prediction_types}) then recalibrating. 
In fact, in our experiments we find that even when learners are of similar power, distribution predictors are not necessarily the most effective choice of base predictor (see Section \ref{sec:experiment}).

\begin{table}[]
    \centering
    \begin{tabular}{p{0.55in}p{0.65in}p{1.65in}}
    \toprule
    Prediction Type & Output Space ($\Rc$) & Interpretation \\
    \midrule
      Point & $\Rbb$ & e.g. estimate of the mean. \\
      Interval & $ \Rbb^2$ & Interval is $[f_1(x), f_2(x)]$. \\
      Quantile & $ \Rbb^K$ & $f_k(x)$ predicts a quantile $\alpha_k \in (0, 1)$. \\
      Distribution & $\Fc(\Rbb)$ & $f[x]$ is a predicted CDF for $y$. \\
      Ensemble & $\Rc_1{\times} {\cdots}{\times}R_K$ & Each $f_k$ is the prediction of a model $k$. \\
     \bottomrule
    \end{tabular}
    \caption{A collection of common prediction types.}
    \label{tab:prediction_types}
\end{table}

\subsection{The calibration score} 
\label{sec:cal_scores}
In this section, we introduce calibration scores for a few prediction types (see Table \ref{tab:prediction_types}) to illustrate the role of the calibration score. Intuitively, a good calibration score should measure how large $y$ is relative to the prediction. Recall that the calibration score $\calfunc: \Rc \times \Yc \to \Rbb$ can be any function that is non-decreasing in $y$. A poor choice of calibration score still guarantees calibration (see Theorem \ref{thm:calibration}), but can harm other metrics such as sharpness or NLL. 
Additional calibration scores for quantile prediction and ensemble prediction can be found in Appendix \ref{app:calscores}.

\paragraph{Point Prediction} A natural calibration score for point predictors is the residue
$
    \calfunc_{\text{residue}}(x, y) = y - f(x)
$. 

\paragraph{Interval Prediction} For interval predictors, a natural choice for the calibration score is the residue divided by the interval size %
$
    \calfunc_\text{interval}(x, y) = (y - f_1(x))/(f_2(x) - f_1(x)) %
$. Intuitively, if $y$ equals the predicted upper bound $f_2(x)$, then the calibration score is $1$; if $y$ equals the lower bound $f_1(x)$ then the calibration score is $0$. The calibration scores of all other $y$ are linear interpolations of these two. 

\paragraph{Distribution Prediction} Given a distribution prediction, i.e. a map $f: \Xc \to \Fc(\Yc)$ two natural choices for the calibration score are
\begin{align*}
    \calfunc_\text{cdf}(x, y) &= f[x](y) \\
    \calfunc_\text{z-score}(x, y) &= (y - \mathrm{mean}(f(x)))/\mathrm{std}(f(x))
\end{align*}
Numerical stability is a practical issue for $\calfunc_\text{cdf}$. When $y$ is small or large, the calibration score may be the same for different $y$ due to rounding with finite numerical precision. 
Empirically, $\calfunc_\text{z-score}$ has better numerical stability and often better performance. %

\subsection{The Interpolation Algorithm} 
Lastly, we discuss the choice of interpolation algorithm. We illustrate a simple linear interpolation algorithm, a randomized interpolation algorithm with strong theoretical guarantees, and a more complex approach using neural autoregressive flows. Recall that an interpolation algorithm is a function $\intalg$ that maps a vector $(u_1, \dots, u_n)$ to a non-decreasing function $\intfunc$ such that $\intfunc(u_1), \dots, \intfunc(u_n)$ are approximately evenly spaced on the unit interval. Recall also that we write $u_{(i)}$ to denote the $i$-th smallest input.

\paragraph{Naive Discretization} As we discussed, the interpolation algorithm in Example 1 is 
\begin{align} 
\intfunc_\text{naive}(u) = i/n, \quad \text{if } u \in [u_{(i)}, u_{(i+1)})
\end{align} 
While simple, the resulting CDF is not continuous, making quantities such as the log-likelihood undefined. It is also not $0$-accurate (recall Definition \ref{def:accurate}). %
For better performance we need more sophisticated interpolation algorithms. 

\paragraph{Linear} Linear interpolation is a simple way to get a continuous CDF function with a density. %
\begin{align*}
    \intfunc_{\mathrm{linear}}(u) = \frac{i + (u-u_{(i)})/(u_{(i+1)} - u_{(i)})}{n+1}
\end{align*}
for $u \in [u_{(i)}, u_{(i+1)})$. A piecewise linear CDF is differentiable almost everywhere, so the log likelihood and density function are well-defined almost everywhere. Linear interpolation can perfectly fit any monotonic sequence, and is therefore $0$-accurate. 

\paragraph{Neural Autoregressive Flow (NAF)} To achieve even better smoothness properties, we can use a neural autoregressive flow (NAF), which is a class of deep neural networks that can universally approximate bounded continuous monotonic functions~\citep{huang2018neural}. The benefit of using a NAF is that the resulting CDF will be more ``smooth''. In fact, if we use a differentiable activation function for the NAF network (such as sigmoid rather than ReLU), then NAF represents smooth CDF functions that are differentiable everywhere. 

The short-coming is that NAF can only represent arbitrary monotonic functions (and hence be a $0$-accurate interpolation algorithm) if the network is infinitely wide. In our experiments, a network with 200 units is sufficient to push the errors below numerical precision.

\paragraph{Random} Finally there is an interpolation algorithm that uses randomization, which would recover the algorithm in \citep{vovk2020conformal}. Let $U$ be uniform on $[0, 1]$. %
\begin{align*}
    \intfunc_{\mathrm{random}}(u) = (i + U)/(n+1) \qquad \text{if } u \in [u_{(i)}, u_{(i+1)}) %
\end{align*}
Compared to linear and NAF, random interpolation has some shortcomings: the CDF is not continuous and the standard deviation is undefined. However, random interpolation has an important theoretical advantage, in that it guarantees that Algorithm \ref{alg:general} is 0-calibrated (i.e., perfectly calibrated). 
In our experiments, this theoretical advantage does not lead to lower calibration error in general, as all methods have near zero ECE. A detailed comparison of the interpolation algorithms in shown in Figure~\ref{fig:interpolation}.

\section{Towards Unifying Calibrated Regression}
\label{sec:generalization}

In this section, we show that modular conformal calibration recovers popular methods for calibrated regression. This implies that the calibration guarantees in this paper also apply to the methods discussed in this section. We also hope to shed light on connections between previously distinct streams of research.

We first observe that isotonic calibration~\citep{kuleshov2018accurate,malik2019calibrated} is recovered by MCC.

\begin{observation}[On Isotonic Calibration]
  Algorithm 1 in \citep{kuleshov2018accurate} is equivalent to Algorithm~\ref{alg:general} in our paper with a distribution base predictor, $\calfunc_\text{cdf}$ and $\intfunc_{\mathrm{linear}}$. 
\end{observation}

Interestingly, this allows us to give new guarantees on the performance of Algorithm 1 in \citep{kuleshov2018accurate}. In particular, we can use Theorem~\ref{thm:calibration} and conclude that Algorithm 1 in \citep{kuleshov2018accurate} is $1/(n+1)$-calibrated. This result was not available in \citep{kuleshov2018accurate}.

\begin{observation}[On Conformal Calibration]
  Algorithm 1 in \citet{vovk2020conformal} is equivalent to Algorithm~\ref{alg:general} in our paper with a distribution base predictor, $\calfunc_\text{cdf}$ and $\intfunc_{\mathrm{random}}$. 
\end{observation}

This makes it clear that conformal calibration and isotonic calibration are tightly connected. The most significant difference between the two methods is that conformal calibration uses a randomized interpolation algorithm. Randomization gives better calibration guarantees at the cost of worse behavior for the distribution predictions (e.g., the predicted distributions are  discontinuous so the log likelihood is ill-defined).

\subsection{Connection to Conformal Interval Prediction}

Conformal prediction~\citep{vovk2005algorithmic,shafer2008tutorial,romano2019conformalized} is a family of (provably) exact interval forecasting algorithms (see, e.g., Proposition \ref{prop:conformal_valid} in Appendix \ref{app:proofs}). Conformal interval prediction uses a proper non-conformity score $\confunc: \Xc \times \Yc \to \Rbb$, which is any continuous function that is strictly unimodal in $y$ (see Appendix \ref{app:proper}). Intuitively, the non-conformity score measures how well the label $y$ matches the input $x$. For example, given a base point prediction function $f: \Xc \to \Rbb$ the absolute residue of the prediction $\confunc(x, y) = |y - f(x)|$ is a natural choice~\citep{vovk2005algorithmic}. 
For a confidence level $c \in (0, 1)$, the conformal forecast is defined as
\begin{align}
    \label{eq:conformal_interval_prediction}
    &I_c(X_*) \\
    &= \left\lbrace y \in \Yc \mid \frac{1}{n} \sum_{i=1}^n \indic \{ \confunc(X_i, Y_i)   \leq \confunc(X_*, y) \}  \leq c   \right\rbrace \nonumber
\end{align}

On the other hand, one can trivially construct a valid confidence interval from a calibrated distribution predictor. Consider the map $\eta_c: \Fc(\Yc) \to \Rbb^2$, which maps any CDF into two numbers that represent a $c$-credible interval.
\begin{align*} 
\eta_c: H[x] \mapsto H[x]^{-1} ( (1 + c) / 2), H[x]^{-1}( (1 - c) / 2)
\end{align*} 
Intuitively, $\eta_c$ returns an interval that has $c$ probability under the distribution $H[x]$. We then ask: \emph{Can modular conformal calibration yield comparable interval predictions to conformal interval prediction?} We answer this question in the affirmative, both theoretically and empirically.

\begin{restatable}{theorem}{thmequivalence}
\label{thm:equivalence}
For the conformal interval predictor $I_c$ with proper non-conformity score, there exists a calibration score $\calfunc$, such that the distribution predictor $H$ given by MCC with calibration score $\calfunc$ and any $0$-exact interpolation algorithm satisfies  
\begin{align}
    \label{eq:conformal_equivalence}
    H[X](U) - H[X](L) \in c \pm \frac{1-c}{n+1} \quad \text{a.s.} %
\end{align}
where $L, U$ are lower/upper bounds of the interval $I_c(X)$. %
\end{restatable} 
See Appendix \ref{app:proofs} for a proof. Theorem \ref{thm:equivalence} states that the conformal prediction interval $[L, U]$ is also a $c$ credible interval (up to $(1-c)/(n+1)$ error) of a distribution prediction made by MCC. In other words, if we know the distribution predicted by the appropriate MCC algorithm, then we can construct the conformal prediction interval by taking a $c$ credible interval.

\begin{table*}[t]
    \centering
    \begin{tabular}{lllll}
\toprule
{} &                                      \textbf{STD} &                            \textbf{95\% CI Width} &                                      \textbf{NLL} &                                     \textbf{CRPS} \\
\midrule
\textsc{zscore-NAF}    &                                 $0.442 \pm 0.003$ &                                 $1.874 \pm 0.037$ &  $\cellcolor{green!25}{\mathbf{0.297 \pm 0.022}}$ &                                 $0.232 \pm 0.002$ \\
\textsc{zscore-linear} &  $\cellcolor{green!25}{\mathbf{0.435 \pm 0.003}}$ &                                 $1.766 \pm 0.016$ &                                 $0.534 \pm 0.021$ &                                 $0.232 \pm 0.002$ \\
\textsc{zscore-random} &                                 $0.438 \pm 0.003$ &                                 $1.776 \pm 0.016$ &                                 N/A &                                 $0.232 \pm 0.002$ \\
\textsc{cdf-NAF}       &                                 $0.446 \pm 0.005$ &  $\cellcolor{green!25}{\mathbf{1.723 \pm 0.027}}$ &                                 $0.465 \pm 0.144$ &                                 $0.245 \pm 0.007$ \\
\textsc{cdf-linear}*    &                                 $0.562 \pm 0.032$ &                                 $1.851 \pm 0.033$ &                                 $0.433 \pm 0.017$ &                                 $0.233 \pm 0.002$ \\
\textsc{cdf-random}*    &                                 $0.587 \pm 0.058$ &                                 $1.851 \pm 0.033$ &                                 N/A &  $\cellcolor{green!25}{\mathbf{0.217 \pm 0.000}}$ \\
\bottomrule
\end{tabular}

    \caption{A comparison of calibration scores and interpolation algorithms when the base predictor is a distribution prediction (* indicates an existing algorithm we compare against). Note that \textsc{CDF-Linear} corresponds to the isotonic recalibration baseline and \textsc{CDF-Random} corresponds to the conformal calibration baseline.  Disaggregated experimental results are shown in Appendix \ref{app:full_experiments}.
    }
    \label{tab:interpolators}
\end{table*}

\begin{table*}[t]
    \centering
    \begin{tabular}{lllll}
\toprule
{} &                                      \textbf{STD} &                            \textbf{95\% CI Width} &                                       \textbf{NLL} &                                     \textbf{CRPS} \\
\midrule
\textsc{point}        &                                 $0.467 \pm 0.006$ &                                 $1.927 \pm 0.016$ &                                  $0.611 \pm 0.017$ &                                 $0.242 \pm 0.002$ \\
\textsc{interval}     &                                 $0.830 \pm 0.336$ &                                 $1.832 \pm 0.034$ &                                 $-0.051 \pm 0.025$ &                                 $0.256 \pm 0.002$ \\
\textsc{quantile-2}   &                                 $0.449 \pm 0.004$ &                                 $1.790 \pm 0.019$ &                                 $-0.101 \pm 0.019$ &                                 $0.228 \pm 0.002$ \\
\textsc{quantile-4}   &                                 $0.439 \pm 0.003$ &                                 $1.692 \pm 0.016$ &  $\cellcolor{green!25}{\mathbf{-0.109 \pm 0.027}}$ &  $\cellcolor{green!25}{\mathbf{0.226 \pm 0.002}}$ \\
\textsc{quantile-7}   &                                 $0.434 \pm 0.003$ &                                 $1.629 \pm 0.015$ &                                 $-0.103 \pm 0.021$ &  $\cellcolor{green!25}{\mathbf{0.226 \pm 0.002}}$ \\
\textsc{quantile-10}  &  $\cellcolor{green!25}{\mathbf{0.432 \pm 0.002}}$ &  $\cellcolor{green!25}{\mathbf{1.625 \pm 0.012}}$ &                                 $-0.042 \pm 0.032$ &  $\cellcolor{green!25}{\mathbf{0.226 \pm 0.002}}$ \\
\textsc{ensemble}     &                                 $0.491 \pm 0.009$ &                                 $1.795 \pm 0.021$ &                                  $0.384 \pm 0.017$ &                                 $0.227 \pm 0.002$ \\
\textsc{distribution} &                                 $0.562 \pm 0.032$ &                                 $1.851 \pm 0.033$ &                                  $0.433 \pm 0.017$ &                                 $0.233 \pm 0.002$ \\
\bottomrule
\end{tabular}
    \caption{A comparison of base predictors. We find that quantile predictors outperform all other prediction types on both sharpness metrics (\metricfont{STD}, \metricfont{95\% CI Width}) and proper scoring rules (\metricfont{NLL}, \metricfont{CRPS}).
    }
    \label{tab:predictors}
\end{table*}

We only know that the conformal prediction interval is \textit{some} credible interval of the distribution prediction, but we don't know \textit{which} credible interval (i.e., $\eta_c$ may not be the correct credible interval). We explore this complicating factor empirically: in particular, we will show that in practice, the conformal interval predictor $I_c$ and the credible interval $\eta_c \circ H[x]$ (with the calibration score $\calfunc$ that is associated with the non-conformity score $\confunc$) have similar performance (see Figure \ref{fig:compare_interval}).

\section{Empirical Study of Recalibration}
\label{sec:experiment}

Our framework introduces three decisions when choosing a recalibration algorithm: the baseline predictor, the calibration score, and the interpolation algorithm. In this section, we investigate how those choices affect performance. We evaluate each combination of 8 base prediction types and 3 interpolation algorithms across 17 regression tasks with 16 random train/test splits per regression task. We also test all of the calibration scores defined in Section \ref{sec:cal_scores}. In total, we train 7,344 calibrated distribution predictors and evaluate each predictor across 4 metrics for a total of 29,376 model evaluations. We summarize our experimental findings in Table \ref{tab:interpolators}, Table \ref{tab:predictors}, and Figure \ref{fig:compare_interval}.

\paragraph{Datasets} We compare MCC algorithms on 17 tabular regression datasets. Most datasets come from the UCI database \citep{dua2019uci}. 
For each dataset we allocate 60\% of the data to learn the base predictor, 20\% for recalibration and 20\% for testing.

\paragraph{Base Predictors} We compare all five prediction types considered in this paper (see Table \ref{tab:prediction_types}). For each base predictor, we use a simple three layer neural network and optimize it with gradient descent. The different base predictors only differ in the number of output dimensions, and the learning objective (i.e. the learning objective should be a proper scoring rule for that prediction type). We try to make the architectures and optimizers of the base predictors as similar as possible across prediction types to isolate the impact of the choice of prediction type, calibration score, and interpolation algorithm, as opposed to the strength of the base predictor. We compare the following base prediction types: 

For \textsc{point} predictors the output dimension is 1 and we minimize the L2 error. For \textsc{quantile} predictors we use 2, 4, 7, 10 equally spaced quantiles (denoted in the plots as quantile-2, quantile-4, quantile-7, quantile-10). For example, for quantile-4 we predict the $1/8, 3/8, 5/8, 7/8$ quantiles. We optimize the neural network with the pinball loss. For \textsc{interval} predictors we use the same setup as \citep{romano2019conformalized} which is equivalent to quantile regression with $5\%, 95\%$ quantiles. For \textsc{distribution} predictors the output of the neural network is 2 dimensions, and we interpret the two dimension as the mean / standard deviation of a Gaussian. We optimize the neural network with the negative log likelihood.
For \textsc{ensemble} predictors we use the setup in \citep{lakshminarayanan2017simple} and learn an ensemble of Gaussian distribution predictors.

\paragraph{Metrics} We compare five measurements of prediction quality. \metricfont{NLL} is the negative log likelihood of the label under the predicted distribution.
\metricfont{CRPS} is the continuous ranked probability score~\citep{hersbach2000decomposition}.
Compared to \metricfont{NLL}, \metricfont{CRPS} is well-defined even for distributions that do not have a density, while \metricfont{NLL} is undefined for such distributions. \metricfont{STD} is the standard deviation of the predicted distribution, a smaller std corresponds to improved sharpness and is generally preferred (all else held equal). \metricfont{95\% CI Width} is the size of centered 95\% credible intervals given by each distribution prediction %
A smaller interval is better (assuming all else are equal). 
\metricfont{ECE} is the expected calibration error~\citep{kuleshov2018accurate}; we use debiased \metricfont{ECE} which should be zero if the predictions are perfectly calibrated.

\paragraph{Results} We find that different recalibration algorithms perform optimally according to different metrics. This supports the need for flexible design frameworks that apply broadly and can be adjusted to the needs of a particular problem. In general, we find that quantile predictors are very effective base predictors that, perhaps surprisingly, tended to outperform distribution base predictors in our experiments. The findings of our experiments are summarized in Table \ref{tab:predictors}, Table \ref{tab:interpolators} and Figure \ref{fig:compare_interval}.

\begin{figure*}
    \centering
    \includegraphics[width=\linewidth]{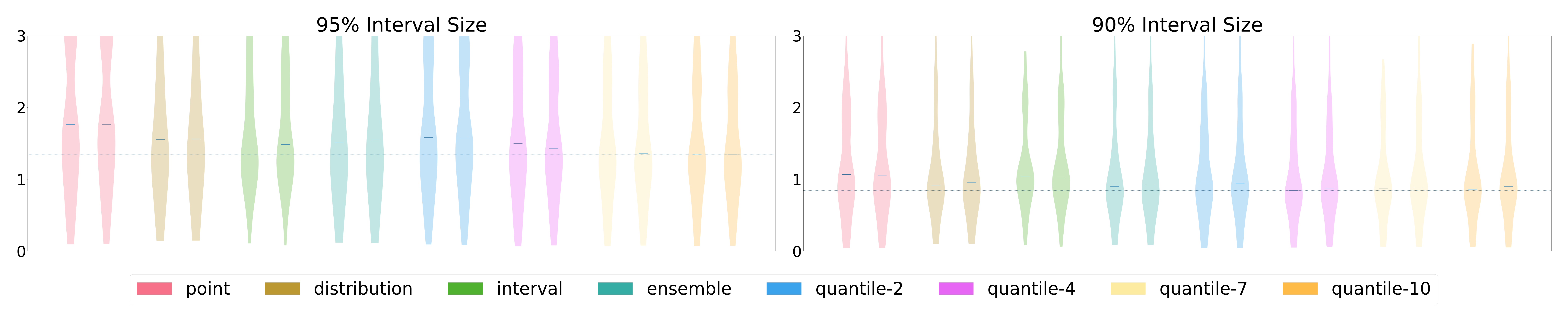}
    \caption{Comparing the interval size from conformal interval prediction (left of each pair) versus credible intervals from recalibrated predictors (right of each pair) for a variety of base prediction types.The intervals from both methods obtain the nominal coverage. The distribution of interval sizes are very similar between the two methods, indicating that recalibration is empirically comparable to conformal interval prediction in its ability to provide interval predictions. }
    \label{fig:compare_interval}
\end{figure*}
\paragraph{On the choice of base predictor}

We find that all base prediction types can be recalibrated to give models with very good calibration. All base predictors we tested achieved an average test ECE of less than 0.007 after recalibration, across the 17 datasets. This is consistent with the calibration guarantee given by our framework, which says that a recalibrated model will be $O(1/n)$-calibrated. 

Quantile predictors performed best on the other four metrics we considered: quantile-10 performed best on the two sharpness metrics \metricfont{STD} and \metricfont{95\% CI Width}, while quantile-4 performed best in terms of \metricfont{NLL} and the three quantile predictors had the same performance on \metricfont{CRPS} (see Table \ref{tab:predictors}). Interestingly, quantile base predictors outperformed distribution estimators on \metricfont{NLL}, even though the distribution predictors were directly trained to optimize \metricfont{NLL}. %
These results indicate that quantile prediction is a promising strategy for learning distribution predictors.

\paragraph{On the choice of calibration score}

We investigate the role of the calibration score for a base predictor that already makes distribution predictions (see Table \ref{tab:interpolators}). Specifically, we compare two natural choices: $\calfunc_\text{cdf}$ and $\calfunc_\text{zscore}$. The $\calfunc_\text{cdf}$ calibration score computes the quantile of the observed label under the predicted distribution, and is the calibration score used by isotonic calibration and conformal calibration. The $\calfunc_\text{cdf}$ calibration score computes the number of standard deviations between the mean of the predicted distribution and the observed label. We find that $\calfunc_\text{z-score}$ and $\calfunc_\text{cdf}$ are effective under different metrics. The $\calfunc_\text{cdf}$ calibration score performs better for \metricfont{CRPS} and \metricfont{95\% CI Width}, while $\calfunc_\text{zscore}$ performs better for \metricfont{NLL} and \metricfont{STD}. %

\paragraph{On the choice of interpolation algorithm}

We compare three interpolation algorithms: Linear interpolation which is simple and stable, random interpolation which provides improved calibration guarantees, and Neural Autoregressive Flow (NAF) interpolation which uses a more sophisticated neural network approach to interpolation. We find that NAF interpolation performs best on \metricfont{NLL} and \metricfont{95\% CI Width}, linear interpolation performs best on \metricfont{STD}, and random interpolation performs best on \metricfont{CRPS}. The random interpolator leads to distribution predictions with infinite \metricfont{STD} and undefine \metricfont{NLL}, so is not appropriate when those metrics are of importance. The most appropriate interpolator is likely to vary between use cases.

\paragraph{On interval prediction}

In this experiment, we explore whether recalibration can yield high quality interval predictions, by comparing to conformal interval prediction, a standard approach for producing interval predictors from any base predictor. Recall that Theorem \ref{thm:equivalence} tells us that conformal interval prediction can approximately be recovered by taking \emph{some} credible interval of a recalibrated predictor. However, since we cannot identify which credible interval it should be a priori, we test in these experiments whether it is sufficient to simply take the centered credible interval; for example, the interval between the 5\% and 95\% quantiles of the predicted distributions. 

We find that this recalibration yields interval predictions that effectively approximate conformal interval prediction (see Figure \ref{fig:compare_interval}). Conformal interval prediction tends to produce slightly shorter intervals than recalibration (both methods achieve the nominal coverage). This shows that recalibration can be applied broadly, even when the downstream task is unknown. If we recalibrate a model to make distribution predictions then decide that we need interval predictions, we can extract credible intervals from the distribution predictor that are comparable to methods designed to directly produce interval predictions.

\section{Discussion}

Recalibration is a convenient and effective way to build calibrated distribution predictors. Flexible methods for uncertainty quantification empower more practitioners to use uncertainty quantification, improving the reliability of both fully-automated systems and decision support systems. Modular conformal calibration organizes and simplifies the process of choosing a recalibration technique, and provides guarantees that the resulting models will be calibrated. As a consequence, we believe that further developing principled and adaptive techniques for choosing between these recalibration algorithms is a promising direction for future work. 

\section*{Acknowledgements} CM is supported by the NSF GRFP. This research was supported by NSF (\#1651565), AFOSR (FA95501910024), ARO (W911NF-21-1-0125) and a Sloan Fellowship.

\bibliographystyle{icml2022}
\bibliography{ref}

\onecolumn 
\newpage
\appendix

\section{Interpolation Algorithms}
\label{app:calscores}
\paragraph{Quantile Prediction} Recall that the interpretation of a quantile prediction is that the probability $y$ is less than $f_k(x)$ should be $\alpha_k$, for $k=1, \dots, K$. %
\begin{align*}
    \calfunc_\text{quantile}(x, y) =
    \left\lbrace 
    \begin{array}{ll} \alpha_K + y - f_K(x) & y > f_K(x)  \\ 
    \alpha_k + \frac{y - f_k(x)}{f_{k+1}(x) - f_k(x) }(\alpha_{k+1} - \alpha_k) 
    & f_k(x) < y \leq f_{k+1}(x)  , \text{for } k=1, \cdots, K-1 \\
    \alpha_1 + y - f_1(x) & y \leq f_1(x) \\
    \end{array} \right. 
\end{align*}

Intuitively, if $y$ exactly equals the $\alpha_k$-th quantile, then $\calfunc_0(x, y) = \alpha_k$. For other values we use a  linear interpolation.

\paragraph{Ensemble Prediction} Given an ensemble consisting of $K$ predictors, we can define the calibration score recursively: for each of the $M$ predictors in the ensemble, we choose a calibration score $\calfunc_m$; the overall calibration score $\calfunc_\text{ensemble}$ is the summed calibration score $\calfunc_\text{ensemble}(x, y) = \sum_k \calfunc_k(x, y)$. Naturally, if there is prior information about the quality of these predictions, we can use a weighted sum where the higher quality predictions are given a higher weight.

\begin{figure*}
    \centering
    \begin{tabular}{c}
        \includegraphics[width=0.4\linewidth]{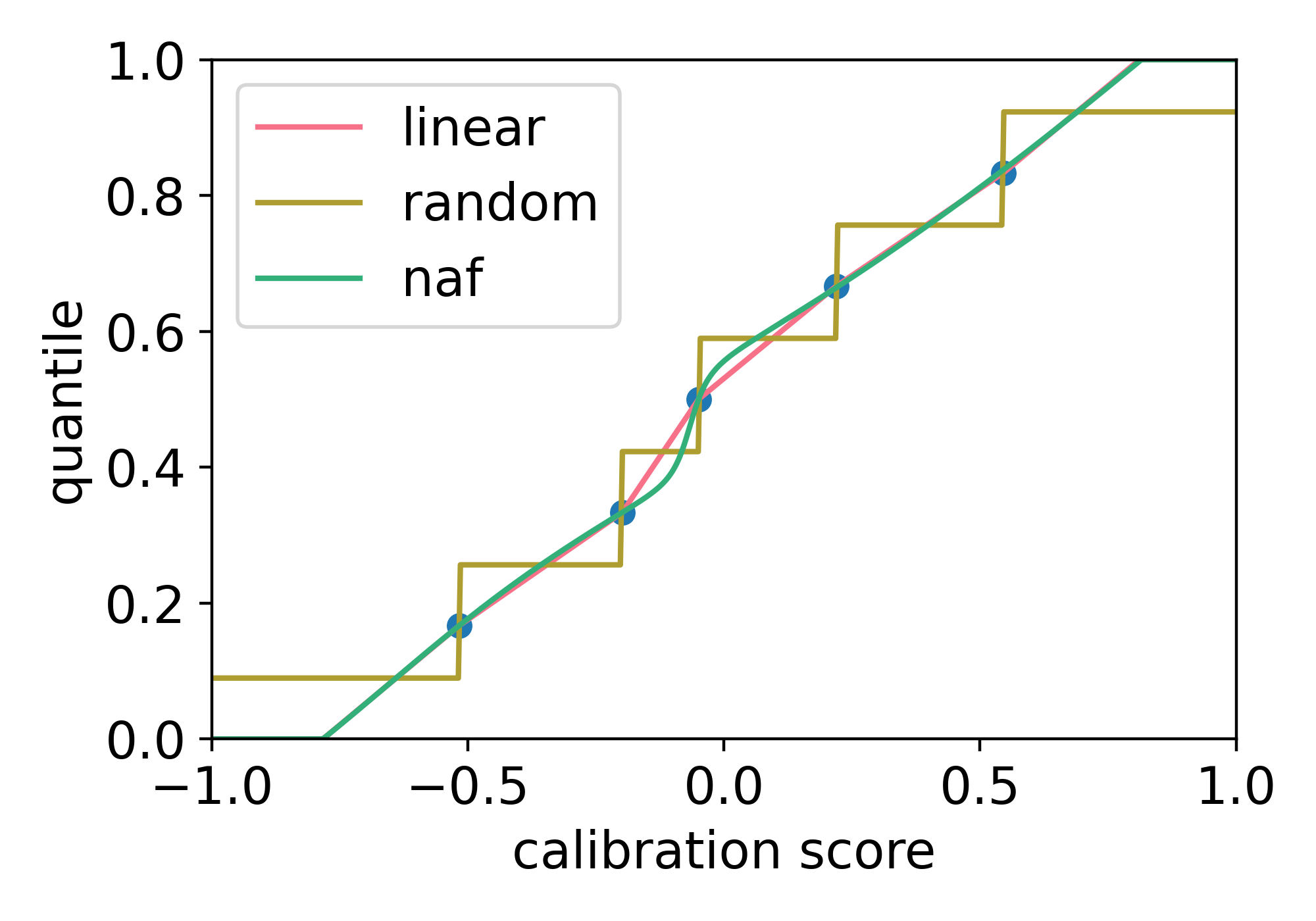} 
        \\
        \begin{tabular}{c|c|c|c}
         & Calibration error & Log-likelihood & Computation \\ \hline
       Random  & Perfect ($\color{green}{+}$) & Undefined ($\color{red}{-}$) & Fast ($\color{green}{+}$) \\
       Linear & $O(1/T)$ ($\color{yellow}{=}$) & Well defined ($\color{yellow}{=}$) & Fast ($\color{green}{+}$)  \\
       NAF & $O(1/T)$ w/ assumptions ($\color{red}{-}$) & Best (empirically) ($\color{green}{+}$) & Slow ($\color{red}{-}$) 
        \end{tabular} 
    \end{tabular}
    \caption{High-level comparison of different interpolation algorithms on different performance benchmarks. \textbf{Top:} A visualization of different interpolation algorithms. Given a set of arbitrary real-valued calibration scores, each interpolation algorithm maps the values to be evenly spaced across the interval [0, 1].  \textbf{Bottom:}$\color{green}{+}$, $\color{yellow}{=}$, $\color{red}{-}$ indicates best, intermediate, or worst.}
    \label{fig:interpolation}
\end{figure*}

\section{Proper Non-conformity Score} 
\label{app:proper}
To ensure that conformal prediction algorithms are ``well-behaved'', it is typical to put some restrictions on the non-conformity score. In particular, we say that a non-conformity score is proper if it is continuous and strictly unimodal in $y$. We require strict unimodality and continuity to ensure that the confidence intervals change smoothly when $c$ increases or decreases. Intuitively, an infinitesimal increase in $c$ should lead to an infinitesimal increase in the confidence interval. 

\section{Proofs} 
\label{app:proofs}

The important property of conformal prediction is that it is always $1/n$-exact (as in Definition 2), regardless of the true distribution of $X,Y$ and the non-conformity score $\confunc$. 

\begin{prop}\label{prop:conformal_valid} For any non-conformity score $\confunc$, if $\confunc(X, Y)$ is absolutely continuous, then  the conformal interval predictor $I_c$ is $1/n$-exact. 
\end{prop} 

\begin{proof}[Proof of Proposition~\ref{prop:conformal_valid} ]
\begin{align*}
    &\Pr[Y \in I_c(X)] = \Pr\left[\frac{1}{n} \# \lbrace i \mid  \confunc(X_i, Y_i) \leq \confunc(X, y) \rbrace \leq c  \right] \\
    &= \Eb\left[ \Pr\left[\frac{1}{n} \# \lbrace i \mid  \confunc(X_i, Y_i) \leq \confunc(X, Y) \rbrace \leq c    \mid \lbag  Z_1, \cdots, Z_n, (X, Y) \rbag \right] \right] \\
    &= \frac{\lfloor nc  \rfloor}{n} > \frac{nc -1}{n} = c - \frac{1}{n}
\end{align*}
where $Z_i=(X_i, Y_i)$.
\end{proof}

\thmcalibration*

\begin{proof}[Proof of Theorem~\ref{thm:calibration}]
By our assumption of absolute continuity, almost surely we have $a'_1 \neq a'_2 \neq \cdots \neq a'_T \neq \calfunc(X^*, Y)$. For notation convenience we also let $a_t = -\infty$ if $t < 1$ and $a_t = +\infty$ if $t > T$.  

By the assumption $\frac{t-\
lambda}{T+1} < q(a'_t) < \frac{t+\lambda}{T+1}$, if $q(a) \geq \frac{t+\lambda}{T+1}$ then $q(a) > q(a_t')$, which by monotonicity implies that $a > a_t'$, i.e. 
\begin{align}
\label{eq:implication1}
\text{if } \intfunc(a) \geq \frac{t+\lambda}{T+1} \text{ then } a > a_t'
\end{align}
similarly 
\begin{align}
\label{eq:implication2}
\text{if } \intfunc(a) \leq \frac{t-\lambda}{T+1} \text{ then } a < a_t'
\end{align} 

\begin{align*} 
&\Pr[H[X^*](Y) \leq c] :=  \Pr[\intfunc(\calfunc(X^*, Y)) \leq c] & \text{Definition} \\
&\leq \Pr\left[ \intfunc(\calfunc(X^*, Y))  \leq \frac{\lceil cT+c+\lambda \rceil - \lambda}{T+1}  \right] &  [i] \\
&\leq \Pr\left[ \calfunc(X^*, Y) < a_{\lceil cT+c+\lambda  \rceil}' \right] &  [i] + \text{Eq.(\ref{eq:implication2})} \\
&= \Ebb\left[ \Pr\left[ \calfunc(X^*, Y) < a_{\lceil cT+c+\lambda \rceil}' \mid \lbag Z_1, \cdots, Z_T, (X^*, Y) \rbag \right] \right]  & \text{Tower} \\
&\leq \lceil cT+c+\lambda  \rceil/(T+1) & \text{Symmetry} 
\end{align*} 
Where explanation $[i]$ is based on the property "$A \implies B \text{ then } \Pr[A] \leq \Pr[B]$"; the last inequality is usually an equality except when $c \approx 1$ then the upper bound will be greater than $1$. 
Similarly we have
\begin{align*} 
&\Pr[H[X_*](Y) \leq c] :=  1 - \Pr[\intfunc(\calfunc(X^*, Y)) > c]  & \text{Definition} \\
&\geq 1 - \Pr\left[ \intfunc(\calfunc(X^*, Y)) \geq \frac{\lfloor  cT + c - \lambda \rfloor + \lambda}{T+1}  \right] & [i] \\
&\geq 1 - \Pr\left[ \calfunc(X^*, Y) > a_{\lfloor cT +c - \lambda \rfloor}' \right]  & [i] + \text{Eq.(\ref{eq:implication1})} \\
&= 1 - \Ebb\left[ \Pr\left[ \calfunc(X^*, Y) > a_{\lfloor cT +c - \lambda \rfloor}' \mid \lbag Z_1, \cdots, Z_T, (X^*, Y) \rbag \right] \right] & \text{Tower} \\
&= \Ebb\left[ 1 - \Pr\left[ \calfunc(X^*, Y) > a_{\lfloor cT + c- \lambda \rfloor}' \mid \lbag Z_1, \cdots, Z_T, (X^*, Y) \rbag \right] \right] & \text{Linear} \\
&= \Ebb\left[ \Pr\left[ \calfunc(X^*, Y) \leq a_{\lfloor cT +c- \lambda \rfloor}' \mid \lbag Z_1, \cdots, Z_T, (X^*, Y) \rbag \right] \right] & \text{Linear} \\
&\geq \lfloor cT +c- \lambda \rfloor/(T+1)  & \text{Symmetry} 
\end{align*} 
Therefore we have 
\begin{align*}
    \Pr[H[X_*](Y) \leq c]  - c &\leq \frac{\lceil cT+c+\lambda  \rceil}{T+1}  - c < \frac{cT+c+\lambda +1}{T+1}  - c = \frac{1 + \lambda} {T+1} \\
    \Pr[H[X_*](Y) \leq c] - c &\geq \frac{\lfloor cT +c- \lambda \rfloor}{T+1}  - c > \frac{cT+c-\lambda -1}{T+1}  - c = \frac{-1  - \lambda} {T+1}
\end{align*}
So combined we have (for any $c \in (0, 1)$) 
\begin{align*}
    \left\lvert \Pr[H[X_*](Y) \leq c]  - c  \right\rvert \leq   \frac{1 + \lambda } {T+1}
\end{align*}
\end{proof}

\thmequivalence*

\begin{figure}
    \centering
    \includegraphics[width=0.7\linewidth]{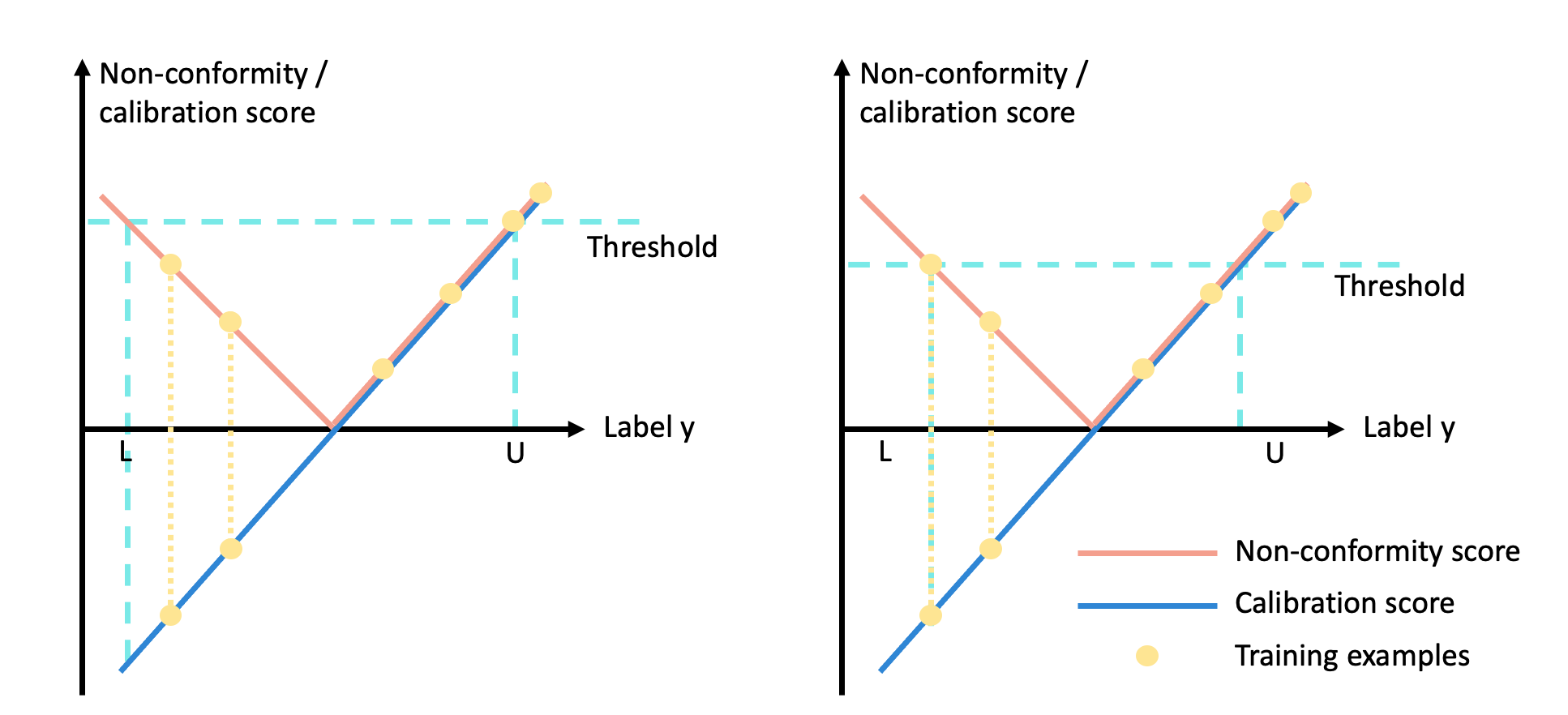}]
    \caption{Illustration of the proof of Theorem~\ref{thm:equivalence}}
    \label{fig:proof_illustration}
\end{figure}

\begin{proof}[Proof of Theorem~\ref{thm:equivalence}]
We will use a constructive proof. Because conformal prediction algorithm does not change if we add a constant to the non-conformity score, so without loss of generality, assume $0$ is a lower bound on $\confunc$. Denote $y_{\mathrm{min}}(\confunc; x)$ as a global minimizer of $\confunc(x, \cdot)$, i.e. 
\begin{align}
    \confunc(x, y_{\mathrm{min}}(\confunc; x)) \leq \confunc(x, y), \forall y \in \Yc
\end{align}
Define 
\begin{align*}
    \calfunc(x, y) = \left\lbrace \begin{array}{ll} -\confunc(x, y) & y \leq y_{\mathrm{min}}(\confunc; x) \\
    \confunc(x, y) & y > y_{\mathrm{min}}(\confunc; x) \end{array} \right.
\end{align*}
Based on this construction we have $\confunc(x, y) = |\calfunc(x, y)|$. In addition because $\confunc$ is uni-modal, $\calfunc$ is monotonically non-decreasing, so $\calfunc$ satisfies the condition as a calibration score. As a notation convenience we will also denote $\calfunc(x, y) = \calfunc_x(y)$. %

Consider the conformal interval predictor (for notation convenience we will drop its dependence on $Z_1, \cdots, Z_T, X^*$)
\begin{align*}
    \finterval := (L, U) &= \left\lbrace y \mid \frac{1}T \#  \lbrace t \mid  |\calfunc_{X_t}(Y_t)| \leq |\calfunc_{X^*}(y)| \rbrace \leq c \right\rbrace \\
\end{align*}
First of all, observe that because $A$ is continuous, we must $|\calfunc_{X^*}(U)| = |\calfunc_{X^*}(L)|$. This intuition is illustrated in Figure~\ref{fig:proof_illustration}. Therefore 
\begin{align*}
     \lbrace \# t \mid  |\calfunc_{X_t}(Y_t)| \leq |\calfunc_{X^*}(L)| \rbrace &= \lbrace \# t \mid  |\calfunc_{X_t}(Y_t)| \leq |\calfunc_{X^*}(U)| \rbrace  \\
     &= \lbrace \# t \mid \calfunc_{X^*}(L) \leq \calfunc_{X_t}(Y_t) \leq \calfunc_{X^*}(U) \rbrace 
\end{align*}

Second we wish to prove that 
\begin{align*} 
 cT    \leq \# \lbrace t \mid |\calfunc_{X_t}(Y_t)| \leq |\calfunc_{X^*}(U)| \rbrace \leq  cT  + 1
 \end{align*} 
This is because if $\# \lbrace t \mid |\calfunc_{X_t}(Y_t)| \leq |\calfunc_{X^*}(U)| \rbrace < cT$ then because of continuity of $\confunc$, almost surely choosing $U' = U + \kappa$ for sufficiently small $\kappa > 0$ still satisfies $\# \lbrace t \mid |\calfunc_{X_t}(Y_t)| \leq |\calfunc_{X^*}(U')| \rbrace < cT$, therefore $U' \in (L, U)$ but $U' > U$, which is a contradiction. 
 
 If on the other hand, $\# \lbrace t \mid |\calfunc_{X_t}(Y_t)| \leq |\calfunc_{X^*}(U)| \rbrace > cT+1$, then let $U' = U - \kappa$ for sufficiently small $\kappa > 0$ we have $\# \lbrace t \mid |\calfunc_{X_t}(Y_t)| \leq |\calfunc_{X^*}(U)| \rbrace > cT$. This means that $U' \not\in (L, U)$ but $U' < U$, which is a contradiction. 
 
We observe that there are two possibilities, these two situations are illustrated in Figure~\ref{fig:proof_illustration}: situation 1. there exists a $t$ such that $\calfunc_{X_t}(Y_t) = \calfunc_{X^*}(U)$; situation 2. there exists a $t$ such that $\calfunc_{X_t}(Y_t) = \calfunc_{X^*}(L)$.

We first consider situation 1. Denote $D = \# \lbrace t, \calfunc_{X^*}(L) < \calfunc_{X_t}(Y_t) \leq \calfunc_{X^*}(U) \rbrace $ and $B = \# \lbrace t, \calfunc_{X_t}(Y_t) \leq \calfunc_{X^*}(L) \rbrace$. We know that $cT \leq D < cT+1$. Then by the assumption that the interpolation algorithm is $0$-exact we have
\begin{align*}
    H[X_*](U) = \frac{D+B+1}{T+1}, H[X](L) \in \left[\frac{B+1}{T+1}, \frac{B+2}{T+1} \right) 
\end{align*}
So their difference is bounded by 
\begin{align*}
    \frac{D-1}{T+1} &< H[X_*](U) - H[X_*](L)  \leq \frac{D}{T+1}  \numberthis\label{eq:bound_diff1}  \\
    \frac{cT-1}{T+1} &< H[X_*](U) - H[X_*](L)  < \frac{cT+1}{T+1}
\end{align*}
Therefore 
\begin{align*}
    H[X_*](U) - H[X_*](L) - c &\leq \frac{cT+1}{T+1} - c = \frac{1-c}{T+1} \\
    c - H[X_*](U) - H[X_*](L) &\geq c - \frac{cT-1}{T+1} = \frac{c-1}{T+1} 
\end{align*}
Combined we have 
\begin{align*}
    H[X](U) - H[X](L) \in c \pm \frac{|1-c|}{T+1}
\end{align*} 

Now we consider situation 2. Denote $D' = \# \lbrace t, \calfunc_{X^*}(L) \leq \calfunc_{X_t}(Y_t) < \calfunc_{X^*}(U) \rbrace $ and $B' = \# \lbrace t, \calfunc_{X_t}(Y_t) < \calfunc_{X^*}(L) \rbrace$. Again we know that $cT \leq D' < cT+1$. Then by the assumption that the interpolation algorithm is $0$-exact we have
\begin{align*}
    H[X_*](U) \in \left[\frac{D'+B'+1}{T+1}, \frac{D'+B'+2}{T+1}\right), H[X_*](L) = \frac{B'+2}{T+1} 
\end{align*}
So their difference is bounded by 
\begin{align*}
    \frac{D'-1}{T+1} &< H[X_*](U) - H[X_*](L)  \leq \frac{D'}{T+1} \numberthis\label{eq:bound_diff2} 
\end{align*}
This is identical to Eq.(\ref{eq:bound_diff1}) so the rest of the proof will follow identically. 
\end{proof} 

\section{Additional Experimental Results}
\label{app:full_experiments}
\begin{table}[]
    \centering
    {\scriptsize \begin{tabular}{lllllll}
\toprule
                         &                        &                             \textbf{\textbf{STD}} &                   \textbf{\textbf{95\% CI Width}} &                              \textbf{\textbf{NLL}} &                            \textbf{\textbf{CRPS}} &                             \textbf{\textbf{ECE}} \\
\midrule
\texttt{blog} & \textsc{zscore-NAF} &  $\cellcolor{green!25}{\mathbf{0.553 \pm 0.006}}$ &                                 $2.737 \pm 0.151$ &                                  $0.178 \pm 0.039$ &  $\cellcolor{green!25}{\mathbf{0.287 \pm 0.003}}$ &                                 $0.002 \pm 0.001$ \\
                         & \textsc{cdf-linear} &                                 $1.380 \pm 0.513$ &                                 $3.074 \pm 0.475$ &   $\cellcolor{green!25}{\mathbf{0.064 \pm 0.062}}$ &                                 $0.289 \pm 0.002$ &  $\cellcolor{green!25}{\mathbf{0.001 \pm 0.001}}$ \\
                         & \textsc{cdf-NAF} &                                 $0.592 \pm 0.061$ &  $\cellcolor{green!25}{\mathbf{2.027 \pm 0.201}}$ &                                  $0.944 \pm 1.229$ &                                 $0.328 \pm 0.049$ &                                 $0.039 \pm 0.028$ \\
                         & \textsc{zscore-random} &                                 $0.567 \pm 0.007$ &                                 $2.657 \pm 0.101$ &                                  $4.087 \pm 0.102$ &                                 $0.288 \pm 0.003$ &                                 $0.002 \pm 0.001$ \\
                         & \textsc{cdf-random} &                                 $1.381 \pm 0.513$ &                                 $3.074 \pm 0.475$ &                                  $4.069 \pm 0.096$ &                                 $0.289 \pm 0.002$ &  $\cellcolor{green!25}{\mathbf{0.001 \pm 0.001}}$ \\
                         & \textsc{zscore-linear} &                                 $0.567 \pm 0.007$ &                                 $2.657 \pm 0.101$ &                                  $0.321 \pm 0.040$ &                                 $0.288 \pm 0.003$ &                                 $0.002 \pm 0.001$ \\
\midrule \texttt{boston} & \textsc{zscore-NAF} &                                 $0.337 \pm 0.019$ &                                 $1.589 \pm 0.195$ &                                  $0.467 \pm 0.117$ &  $\cellcolor{green!25}{\mathbf{0.173 \pm 0.012}}$ &  $\cellcolor{green!25}{\mathbf{0.009 \pm 0.010}}$ \\
                         & \textsc{cdf-linear} &                                 $0.372 \pm 0.031$ &                                 $1.527 \pm 0.118$ &                                  $0.642 \pm 0.075$ &  $\cellcolor{green!25}{\mathbf{0.173 \pm 0.012}}$ &  $\cellcolor{green!25}{\mathbf{0.009 \pm 0.009}}$ \\
                         & \textsc{cdf-NAF} &                                 $0.357 \pm 0.025$ &                                 $1.539 \pm 0.124$ &   $\cellcolor{green!25}{\mathbf{0.363 \pm 0.070}}$ &                                 $0.174 \pm 0.012$ &  $\cellcolor{green!25}{\mathbf{0.009 \pm 0.010}}$ \\
                         & \textsc{zscore-random} &                                 $0.335 \pm 0.018$ &                                 $1.520 \pm 0.121$ &                                 $13.622 \pm 0.094$ &  $\cellcolor{green!25}{\mathbf{0.173 \pm 0.012}}$ &  $\cellcolor{green!25}{\mathbf{0.009 \pm 0.009}}$ \\
                         & \textsc{cdf-random} &                                               N/A &                                 $1.532 \pm 0.121$ &                                 $13.588 \pm 0.114$ &                                               N/A &  $\cellcolor{green!25}{\mathbf{0.009 \pm 0.009}}$ \\
                         & \textsc{zscore-linear} &  $\cellcolor{green!25}{\mathbf{0.331 \pm 0.017}}$ &  $\cellcolor{green!25}{\mathbf{1.449 \pm 0.111}}$ &                                  $0.674 \pm 0.068$ &  $\cellcolor{green!25}{\mathbf{0.173 \pm 0.012}}$ &  $\cellcolor{green!25}{\mathbf{0.009 \pm 0.009}}$ \\
\midrule \texttt{concrete} & \textsc{zscore-NAF} &                                 $0.325 \pm 0.017$ &                                 $1.402 \pm 0.186$ &                                  $0.325 \pm 0.066$ &  $\cellcolor{green!25}{\mathbf{0.163 \pm 0.007}}$ &  $\cellcolor{green!25}{\mathbf{0.005 \pm 0.005}}$ \\
                         & \textsc{cdf-linear} &                                 $0.317 \pm 0.016$ &                                 $1.287 \pm 0.086$ &                                  $0.576 \pm 0.054$ &  $\cellcolor{green!25}{\mathbf{0.163 \pm 0.007}}$ &  $\cellcolor{green!25}{\mathbf{0.005 \pm 0.005}}$ \\
                         & \textsc{cdf-NAF} &                                 $0.312 \pm 0.016$ &                                 $1.354 \pm 0.078$ &   $\cellcolor{green!25}{\mathbf{0.251 \pm 0.051}}$ &  $\cellcolor{green!25}{\mathbf{0.163 \pm 0.007}}$ &  $\cellcolor{green!25}{\mathbf{0.005 \pm 0.006}}$ \\
                         & \textsc{zscore-random} &                                 $0.310 \pm 0.014$ &                                 $1.287 \pm 0.087$ &                                 $13.206 \pm 0.090$ &  $\cellcolor{green!25}{\mathbf{0.163 \pm 0.007}}$ &  $\cellcolor{green!25}{\mathbf{0.005 \pm 0.005}}$ \\
                         & \textsc{cdf-random} &                                               N/A &                                 $1.288 \pm 0.086$ &                                 $13.212 \pm 0.100$ &                                               N/A &  $\cellcolor{green!25}{\mathbf{0.005 \pm 0.005}}$ \\
                         & \textsc{zscore-linear} &  $\cellcolor{green!25}{\mathbf{0.309 \pm 0.014}}$ &  $\cellcolor{green!25}{\mathbf{1.279 \pm 0.084}}$ &                                  $0.593 \pm 0.057$ &  $\cellcolor{green!25}{\mathbf{0.163 \pm 0.007}}$ &  $\cellcolor{green!25}{\mathbf{0.005 \pm 0.005}}$ \\
\midrule \texttt{crime} & \textsc{zscore-NAF} &                                 $0.498 \pm 0.013$ &                                 $2.062 \pm 0.084$ &                                  $1.764 \pm 0.123$ &                                 $0.309 \pm 0.008$ &                                 $0.010 \pm 0.006$ \\
                         & \textsc{cdf-linear} &                                 $0.644 \pm 0.029$ &                                 $2.406 \pm 0.066$ &                                  $1.077 \pm 0.057$ &  $\cellcolor{green!25}{\mathbf{0.308 \pm 0.007}}$ &  $\cellcolor{green!25}{\mathbf{0.006 \pm 0.006}}$ \\
                         & \textsc{cdf-NAF} &                                 $0.570 \pm 0.014$ &                                 $2.161 \pm 0.053$ &   $\cellcolor{green!25}{\mathbf{0.814 \pm 0.067}}$ &                                 $0.311 \pm 0.007$ &  $\cellcolor{green!25}{\mathbf{0.006 \pm 0.006}}$ \\
                         & \textsc{zscore-random} &  $\cellcolor{green!25}{\mathbf{0.487 \pm 0.011}}$ &  $\cellcolor{green!25}{\mathbf{1.900 \pm 0.036}}$ &                                 $13.308 \pm 0.071$ &  $\cellcolor{green!25}{\mathbf{0.308 \pm 0.008}}$ &  $\cellcolor{green!25}{\mathbf{0.006 \pm 0.006}}$ \\
                         & \textsc{cdf-random} &                                               N/A &                                 $2.406 \pm 0.066$ &                                 $13.274 \pm 0.054$ &                                               N/A &  $\cellcolor{green!25}{\mathbf{0.006 \pm 0.006}}$ \\
                         & \textsc{zscore-linear} &  $\cellcolor{green!25}{\mathbf{0.487 \pm 0.011}}$ &  $\cellcolor{green!25}{\mathbf{1.900 \pm 0.036}}$ &                                  $1.901 \pm 0.130$ &  $\cellcolor{green!25}{\mathbf{0.308 \pm 0.008}}$ &  $\cellcolor{green!25}{\mathbf{0.006 \pm 0.006}}$ \\
\midrule \texttt{energy} & \textsc{zscore-NAF} &                                 $0.174 \pm 0.013$ &  $\cellcolor{green!25}{\mathbf{0.652 \pm 0.043}}$ &  $\cellcolor{green!25}{\mathbf{-0.394 \pm 0.097}}$ &  $\cellcolor{green!25}{\mathbf{0.099 \pm 0.008}}$ &  $\cellcolor{green!25}{\mathbf{0.010 \pm 0.006}}$ \\
 \texttt{-efficiency}                        & \textsc{cdf-linear} &  $\cellcolor{green!25}{\mathbf{0.168 \pm 0.010}}$ &                                 $0.668 \pm 0.040$ &                                  $0.027 \pm 0.108$ &  $\cellcolor{green!25}{\mathbf{0.099 \pm 0.008}}$ &  $\cellcolor{green!25}{\mathbf{0.010 \pm 0.006}}$ \\
                         & \textsc{cdf-NAF} &  $\cellcolor{green!25}{\mathbf{0.168 \pm 0.010}}$ &                                 $0.681 \pm 0.039$ &                                 $-0.342 \pm 0.101$ &  $\cellcolor{green!25}{\mathbf{0.099 \pm 0.008}}$ &  $\cellcolor{green!25}{\mathbf{0.010 \pm 0.006}}$ \\
                         & \textsc{zscore-random} &                                 $0.180 \pm 0.010$ &                                 $0.667 \pm 0.040$ &                                 $13.253 \pm 0.130$ &  $\cellcolor{green!25}{\mathbf{0.099 \pm 0.008}}$ &  $\cellcolor{green!25}{\mathbf{0.010 \pm 0.006}}$ \\
                         & \textsc{cdf-random} &                                               N/A &                                 $0.668 \pm 0.040$ &                                 $13.252 \pm 0.125$ &                                               N/A &  $\cellcolor{green!25}{\mathbf{0.010 \pm 0.006}}$ \\
                         & \textsc{zscore-linear} &                                 $0.172 \pm 0.011$ &                                 $0.666 \pm 0.040$ &                                  $0.003 \pm 0.107$ &  $\cellcolor{green!25}{\mathbf{0.099 \pm 0.008}}$ &  $\cellcolor{green!25}{\mathbf{0.010 \pm 0.006}}$ \\
\midrule \texttt{fb-comment1} & \textsc{zscore-NAF} &                                 $0.376 \pm 0.009$ &                                 $1.459 \pm 0.028$ &   $\cellcolor{green!25}{\mathbf{0.066 \pm 0.015}}$ &  $\cellcolor{green!25}{\mathbf{0.210 \pm 0.001}}$ &  $\cellcolor{green!25}{\mathbf{0.002 \pm 0.001}}$ \\
                         & \textsc{cdf-linear} &                                 $0.880 \pm 0.119$ &                                 $1.692 \pm 0.034$ &                                  $0.120 \pm 0.017$ &  $\cellcolor{green!25}{\mathbf{0.210 \pm 0.001}}$ &  $\cellcolor{green!25}{\mathbf{0.002 \pm 0.001}}$ \\
                         & \textsc{cdf-NAF} &  $\cellcolor{green!25}{\mathbf{0.344 \pm 0.054}}$ &  $\cellcolor{green!25}{\mathbf{1.156 \pm 0.195}}$ &                                  $2.841 \pm 2.106$ &                                 $0.288 \pm 0.075$ &                                 $0.091 \pm 0.046$ \\
                         & \textsc{zscore-random} &                                 $0.397 \pm 0.004$ &                                 $1.692 \pm 0.034$ &                                  $5.196 \pm 0.030$ &  $\cellcolor{green!25}{\mathbf{0.210 \pm 0.001}}$ &  $\cellcolor{green!25}{\mathbf{0.002 \pm 0.001}}$ \\
                         & \textsc{cdf-random} &                                 $0.880 \pm 0.119$ &                                 $1.692 \pm 0.034$ &                                  $5.220 \pm 0.029$ &  $\cellcolor{green!25}{\mathbf{0.210 \pm 0.001}}$ &  $\cellcolor{green!25}{\mathbf{0.002 \pm 0.001}}$ \\
                         & \textsc{zscore-linear} &                                 $0.397 \pm 0.004$ &                                 $1.692 \pm 0.034$ &                                  $0.261 \pm 0.015$ &  $\cellcolor{green!25}{\mathbf{0.210 \pm 0.001}}$ &  $\cellcolor{green!25}{\mathbf{0.002 \pm 0.001}}$ \\
\midrule \texttt{fb-comment2} & \textsc{zscore-NAF} &                                 $0.368 \pm 0.010$ &                                 $1.423 \pm 0.021$ &   $\cellcolor{green!25}{\mathbf{0.053 \pm 0.025}}$ &  $\cellcolor{green!25}{\mathbf{0.207 \pm 0.002}}$ &  $\cellcolor{green!25}{\mathbf{0.001 \pm 0.001}}$ \\
                         & \textsc{cdf-linear} &                                 $0.459 \pm 0.017$ &                                 $1.573 \pm 0.016$ &                                  $0.118 \pm 0.027$ &  $\cellcolor{green!25}{\mathbf{0.207 \pm 0.002}}$ &  $\cellcolor{green!25}{\mathbf{0.001 \pm 0.001}}$ \\
                         & \textsc{cdf-NAF} &  $\cellcolor{green!25}{\mathbf{0.347 \pm 0.035}}$ &  $\cellcolor{green!25}{\mathbf{1.294 \pm 0.171}}$ &                                                N/A &                                 $0.273 \pm 0.083$ &                                               N/A \\
                         & \textsc{zscore-random} &                                 $0.386 \pm 0.004$ &                                 $1.573 \pm 0.016$ &                                  $3.739 \pm 0.031$ &  $\cellcolor{green!25}{\mathbf{0.207 \pm 0.002}}$ &  $\cellcolor{green!25}{\mathbf{0.001 \pm 0.001}}$ \\
                         & \textsc{cdf-random} &                                 $0.459 \pm 0.017$ &                                 $1.573 \pm 0.016$ &                                  $3.731 \pm 0.037$ &  $\cellcolor{green!25}{\mathbf{0.207 \pm 0.002}}$ &  $\cellcolor{green!25}{\mathbf{0.001 \pm 0.001}}$ \\
                         & \textsc{zscore-linear} &                                 $0.386 \pm 0.004$ &                                 $1.573 \pm 0.016$ &                                  $0.194 \pm 0.023$ &  $\cellcolor{green!25}{\mathbf{0.207 \pm 0.002}}$ &  $\cellcolor{green!25}{\mathbf{0.001 \pm 0.001}}$ \\
\midrule \texttt{forest-fires} & \textsc{zscore-NAF} &                                 $1.167 \pm 0.043$ &                                 $4.697 \pm 0.214$ &                                  $1.959 \pm 0.137$ &  $\cellcolor{green!25}{\mathbf{0.601 \pm 0.023}}$ &                                 $0.018 \pm 0.008$ \\
                         & \textsc{cdf-linear} &                                 $1.285 \pm 0.162$ &                                 $4.874 \pm 0.380$ &                                  $1.964 \pm 0.086$ &  $\cellcolor{green!25}{\mathbf{0.601 \pm 0.022}}$ &  $\cellcolor{green!25}{\mathbf{0.017 \pm 0.008}}$ \\
                         & \textsc{cdf-NAF} &                                 $1.219 \pm 0.078$ &                                 $4.801 \pm 0.339$ &   $\cellcolor{green!25}{\mathbf{1.637 \pm 0.087}}$ &                                 $0.607 \pm 0.023$ &  $\cellcolor{green!25}{\mathbf{0.017 \pm 0.008}}$ \\
                         & \textsc{zscore-random} &                                 $1.156 \pm 0.041$ &                                 $4.529 \pm 0.185$ &                                 $13.648 \pm 0.102$ &  $\cellcolor{green!25}{\mathbf{0.601 \pm 0.023}}$ &  $\cellcolor{green!25}{\mathbf{0.017 \pm 0.008}}$ \\
                         & \textsc{cdf-random} &                                               N/A &                                 $4.862 \pm 0.381$ &                                 $13.659 \pm 0.099$ &                                               N/A &  $\cellcolor{green!25}{\mathbf{0.017 \pm 0.008}}$ \\
                         & \textsc{zscore-linear} &  $\cellcolor{green!25}{\mathbf{1.147 \pm 0.043}}$ &  $\cellcolor{green!25}{\mathbf{4.455 \pm 0.219}}$ &                                  $2.121 \pm 0.124$ &  $\cellcolor{green!25}{\mathbf{0.601 \pm 0.023}}$ &  $\cellcolor{green!25}{\mathbf{0.017 \pm 0.008}}$ \\
\midrule \texttt{kin8nm} & \textsc{zscore-NAF} &                                 $0.300 \pm 0.006$ &  $\cellcolor{green!25}{\mathbf{1.107 \pm 0.016}}$ &                                  $0.127 \pm 0.012$ &  $\cellcolor{green!25}{\mathbf{0.152 \pm 0.002}}$ &  $\cellcolor{green!25}{\mathbf{0.003 \pm 0.002}}$ \\
                         & \textsc{cdf-linear} &  $\cellcolor{green!25}{\mathbf{0.281 \pm 0.003}}$ &                                 $1.121 \pm 0.013$ &                                  $0.490 \pm 0.015$ &  $\cellcolor{green!25}{\mathbf{0.152 \pm 0.002}}$ &  $\cellcolor{green!25}{\mathbf{0.003 \pm 0.002}}$ \\
                         & \textsc{cdf-NAF} &  $\cellcolor{green!25}{\mathbf{0.281 \pm 0.003}}$ &                                 $1.149 \pm 0.015$ &   $\cellcolor{green!25}{\mathbf{0.112 \pm 0.014}}$ &  $\cellcolor{green!25}{\mathbf{0.152 \pm 0.002}}$ &  $\cellcolor{green!25}{\mathbf{0.003 \pm 0.002}}$ \\
                         & \textsc{zscore-random} &  $\cellcolor{green!25}{\mathbf{0.281 \pm 0.003}}$ &                                 $1.121 \pm 0.013$ &                                 $11.279 \pm 0.090$ &  $\cellcolor{green!25}{\mathbf{0.152 \pm 0.002}}$ &  $\cellcolor{green!25}{\mathbf{0.003 \pm 0.002}}$ \\
                         & \textsc{cdf-random} &  $\cellcolor{green!25}{\mathbf{0.281 \pm 0.003}}$ &                                 $1.121 \pm 0.013$ &                                 $11.291 \pm 0.073$ &  $\cellcolor{green!25}{\mathbf{0.152 \pm 0.002}}$ &  $\cellcolor{green!25}{\mathbf{0.003 \pm 0.002}}$ \\
                         & \textsc{zscore-linear} &  $\cellcolor{green!25}{\mathbf{0.281 \pm 0.003}}$ &                                 $1.121 \pm 0.013$ &                                  $0.486 \pm 0.014$ &  $\cellcolor{green!25}{\mathbf{0.152 \pm 0.002}}$ &  $\cellcolor{green!25}{\mathbf{0.003 \pm 0.002}}$ \\
\bottomrule
\end{tabular}}
    \caption{Experimental results for individual datasets.}        \label{tab:my_label}
\end{table}

\begin{table}[]
    \centering
    {\scriptsize \begin{tabular}{lllllll}
\toprule
                        &                        &                             \textbf{\textbf{STD}} &                   \textbf{\textbf{95\% CI Width}} &                              \textbf{\textbf{NLL}} &                            \textbf{\textbf{CRPS}} &                             \textbf{\textbf{ECE}} \\
\midrule
 \texttt{medical} & \textsc{zscore-NAF} &                                 $0.935 \pm 0.008$ &                                 $4.465 \pm 0.064$ &                                  $1.548 \pm 0.015$ &                                 $0.463 \pm 0.002$ &                                 $0.003 \pm 0.001$ \\
\texttt{-expenditure}                        & \textsc{cdf-linear} &                                 $0.963 \pm 0.009$ &                                 $3.648 \pm 0.037$ &                                  $1.643 \pm 0.009$ &  $\cellcolor{green!25}{\mathbf{0.462 \pm 0.002}}$ &  $\cellcolor{green!25}{\mathbf{0.002 \pm 0.001}}$ \\
                        & \textsc{cdf-NAF} &                                 $0.928 \pm 0.010$ &  $\cellcolor{green!25}{\mathbf{3.616 \pm 0.110}}$ &   $\cellcolor{green!25}{\mathbf{1.381 \pm 0.057}}$ &                                 $0.465 \pm 0.002$ &                                 $0.005 \pm 0.004$ \\
                        & \textsc{zscore-random} &  $\cellcolor{green!25}{\mathbf{0.893 \pm 0.004}}$ &                                 $3.648 \pm 0.037$ &                                 $10.869 \pm 0.041$ &  $\cellcolor{green!25}{\mathbf{0.462 \pm 0.002}}$ &  $\cellcolor{green!25}{\mathbf{0.002 \pm 0.001}}$ \\
                        & \textsc{cdf-random} &                                 $0.963 \pm 0.009$ &                                 $3.648 \pm 0.037$ &                                 $10.866 \pm 0.034$ &  $\cellcolor{green!25}{\mathbf{0.462 \pm 0.002}}$ &  $\cellcolor{green!25}{\mathbf{0.002 \pm 0.001}}$ \\
                        & \textsc{zscore-linear} &  $\cellcolor{green!25}{\mathbf{0.893 \pm 0.004}}$ &                                 $3.648 \pm 0.037$ &                                  $1.818 \pm 0.015$ &  $\cellcolor{green!25}{\mathbf{0.462 \pm 0.002}}$ &  $\cellcolor{green!25}{\mathbf{0.002 \pm 0.001}}$ \\
\midrule \texttt{mpg} & \textsc{zscore-NAF} &                                 $0.396 \pm 0.021$ &                                 $1.832 \pm 0.186$ &                                  $0.555 \pm 0.114$ &  $\cellcolor{green!25}{\mathbf{0.186 \pm 0.011}}$ &  $\cellcolor{green!25}{\mathbf{0.019 \pm 0.013}}$ \\
                        & \textsc{cdf-linear} &                                 $0.392 \pm 0.030$ &                                 $1.554 \pm 0.110$ &                                  $0.821 \pm 0.099$ &  $\cellcolor{green!25}{\mathbf{0.186 \pm 0.011}}$ &                                 $0.020 \pm 0.013$ \\
                        & \textsc{cdf-NAF} &                                 $0.383 \pm 0.026$ &                                 $1.652 \pm 0.138$ &   $\cellcolor{green!25}{\mathbf{0.551 \pm 0.084}}$ &                                 $0.187 \pm 0.011$ &  $\cellcolor{green!25}{\mathbf{0.019 \pm 0.013}}$ \\
                        & \textsc{zscore-random} &                                 $0.380 \pm 0.018$ &                                 $1.545 \pm 0.109$ &                                 $13.696 \pm 0.073$ &  $\cellcolor{green!25}{\mathbf{0.186 \pm 0.011}}$ &                                 $0.020 \pm 0.013$ \\
                        & \textsc{cdf-random} &                                               N/A &                                 $1.546 \pm 0.110$ &                                 $13.651 \pm 0.092$ &                                               N/A &                                 $0.020 \pm 0.013$ \\
                        & \textsc{zscore-linear} &  $\cellcolor{green!25}{\mathbf{0.377 \pm 0.018}}$ &  $\cellcolor{green!25}{\mathbf{1.538 \pm 0.108}}$ &                                  $0.756 \pm 0.096$ &  $\cellcolor{green!25}{\mathbf{0.186 \pm 0.011}}$ &  $\cellcolor{green!25}{\mathbf{0.019 \pm 0.013}}$ \\
\midrule \texttt{naval} & \textsc{zscore-NAF} &  $\cellcolor{green!25}{\mathbf{0.041 \pm 0.002}}$ &                                 $0.165 \pm 0.008$ &                                 $-1.904 \pm 0.036$ &  $\cellcolor{green!25}{\mathbf{0.025 \pm 0.001}}$ &  $\cellcolor{green!25}{\mathbf{0.002 \pm 0.001}}$ \\
                        & \textsc{cdf-linear} &                                 $0.042 \pm 0.002$ &  $\cellcolor{green!25}{\mathbf{0.162 \pm 0.007}}$ &                                 $-1.716 \pm 0.029$ &  $\cellcolor{green!25}{\mathbf{0.025 \pm 0.001}}$ &  $\cellcolor{green!25}{\mathbf{0.002 \pm 0.001}}$ \\
                        & \textsc{cdf-NAF} &                                 $0.042 \pm 0.002$ &                                 $0.173 \pm 0.008$ &  $\cellcolor{green!25}{\mathbf{-1.915 \pm 0.027}}$ &  $\cellcolor{green!25}{\mathbf{0.025 \pm 0.001}}$ &  $\cellcolor{green!25}{\mathbf{0.002 \pm 0.001}}$ \\
                        & \textsc{zscore-random} &                                 $0.042 \pm 0.002$ &  $\cellcolor{green!25}{\mathbf{0.162 \pm 0.007}}$ &                                  $1.355 \pm 0.125$ &  $\cellcolor{green!25}{\mathbf{0.025 \pm 0.001}}$ &  $\cellcolor{green!25}{\mathbf{0.002 \pm 0.001}}$ \\
                        & \textsc{cdf-random} &                                 $0.042 \pm 0.002$ &  $\cellcolor{green!25}{\mathbf{0.162 \pm 0.007}}$ &                                  $1.330 \pm 0.114$ &  $\cellcolor{green!25}{\mathbf{0.025 \pm 0.001}}$ &  $\cellcolor{green!25}{\mathbf{0.002 \pm 0.001}}$ \\
                        & \textsc{zscore-linear} &                                 $0.042 \pm 0.002$ &  $\cellcolor{green!25}{\mathbf{0.162 \pm 0.007}}$ &                                 $-1.717 \pm 0.029$ &  $\cellcolor{green!25}{\mathbf{0.025 \pm 0.001}}$ &  $\cellcolor{green!25}{\mathbf{0.002 \pm 0.001}}$ \\
\midrule \texttt{power-plant} & \textsc{zscore-NAF} &  $\cellcolor{green!25}{\mathbf{0.218 \pm 0.003}}$ &  $\cellcolor{green!25}{\mathbf{0.845 \pm 0.007}}$ &  $\cellcolor{green!25}{\mathbf{-0.104 \pm 0.010}}$ &  $\cellcolor{green!25}{\mathbf{0.120 \pm 0.001}}$ &  $\cellcolor{green!25}{\mathbf{0.003 \pm 0.002}}$ \\
                        & \textsc{cdf-linear} &                                 $0.225 \pm 0.003$ &                                 $0.854 \pm 0.009$ &                                  $0.244 \pm 0.011$ &  $\cellcolor{green!25}{\mathbf{0.120 \pm 0.001}}$ &  $\cellcolor{green!25}{\mathbf{0.003 \pm 0.002}}$ \\
                        & \textsc{cdf-NAF} &                                 $0.224 \pm 0.003$ &                                 $0.886 \pm 0.011$ &                                 $-0.101 \pm 0.010$ &                                 $0.121 \pm 0.001$ &  $\cellcolor{green!25}{\mathbf{0.003 \pm 0.002}}$ \\
                        & \textsc{zscore-random} &                                 $0.219 \pm 0.002$ &                                 $0.854 \pm 0.009$ &                                 $10.208 \pm 0.074$ &  $\cellcolor{green!25}{\mathbf{0.120 \pm 0.001}}$ &  $\cellcolor{green!25}{\mathbf{0.003 \pm 0.002}}$ \\
                        & \textsc{cdf-random} &                                 $0.225 \pm 0.003$ &                                 $0.854 \pm 0.009$ &                                 $10.208 \pm 0.067$ &  $\cellcolor{green!25}{\mathbf{0.120 \pm 0.001}}$ &  $\cellcolor{green!25}{\mathbf{0.003 \pm 0.002}}$ \\
                        & \textsc{zscore-linear} &  $\cellcolor{green!25}{\mathbf{0.218 \pm 0.002}}$ &                                 $0.854 \pm 0.009$ &                                  $0.251 \pm 0.011$ &  $\cellcolor{green!25}{\mathbf{0.120 \pm 0.001}}$ &  $\cellcolor{green!25}{\mathbf{0.003 \pm 0.002}}$ \\
\midrule \texttt{protein} & \textsc{zscore-NAF} &                                 $0.622 \pm 0.011$ &                                 $2.246 \pm 0.027$ &   $\cellcolor{green!25}{\mathbf{0.765 \pm 0.019}}$ &  $\cellcolor{green!25}{\mathbf{0.339 \pm 0.003}}$ &  $\cellcolor{green!25}{\mathbf{0.001 \pm 0.001}}$ \\
                        & \textsc{cdf-linear} &                                 $0.709 \pm 0.020$ &                                 $2.333 \pm 0.024$ &                                  $0.960 \pm 0.018$ &  $\cellcolor{green!25}{\mathbf{0.339 \pm 0.003}}$ &  $\cellcolor{green!25}{\mathbf{0.001 \pm 0.001}}$ \\
                        & \textsc{cdf-NAF} &  $\cellcolor{green!25}{\mathbf{0.617 \pm 0.007}}$ &  $\cellcolor{green!25}{\mathbf{2.138 \pm 0.072}}$ &                                  $0.968 \pm 0.048$ &                                 $0.341 \pm 0.003$ &                                 $0.015 \pm 0.006$ \\
                        & \textsc{zscore-random} &                                 $0.627 \pm 0.006$ &                                 $2.333 \pm 0.024$ &                                  $7.366 \pm 0.072$ &  $\cellcolor{green!25}{\mathbf{0.339 \pm 0.003}}$ &  $\cellcolor{green!25}{\mathbf{0.001 \pm 0.001}}$ \\
                        & \textsc{cdf-random} &                                 $0.709 \pm 0.020$ &                                 $2.333 \pm 0.024$ &                                  $7.393 \pm 0.073$ &  $\cellcolor{green!25}{\mathbf{0.339 \pm 0.003}}$ &  $\cellcolor{green!25}{\mathbf{0.001 \pm 0.001}}$ \\
                        & \textsc{zscore-linear} &                                 $0.627 \pm 0.006$ &                                 $2.333 \pm 0.024$ &                                  $0.995 \pm 0.017$ &  $\cellcolor{green!25}{\mathbf{0.339 \pm 0.003}}$ &  $\cellcolor{green!25}{\mathbf{0.001 \pm 0.001}}$ \\
\midrule \texttt{super} & \textsc{zscore-NAF} &                                 $0.307 \pm 0.005$ &                                 $1.140 \pm 0.016$ &  $\cellcolor{green!25}{\mathbf{-0.246 \pm 0.016}}$ &  $\cellcolor{green!25}{\mathbf{0.149 \pm 0.002}}$ &  $\cellcolor{green!25}{\mathbf{0.003 \pm 0.002}}$ \\
\texttt{-conductivity}                        & \textsc{cdf-linear} &                                 $0.343 \pm 0.017$ &                                 $1.179 \pm 0.016$ &                                 $-0.052 \pm 0.013$ &  $\cellcolor{green!25}{\mathbf{0.149 \pm 0.002}}$ &  $\cellcolor{green!25}{\mathbf{0.003 \pm 0.002}}$ \\
                        & \textsc{cdf-NAF} &  $\cellcolor{green!25}{\mathbf{0.294 \pm 0.005}}$ &  $\cellcolor{green!25}{\mathbf{1.089 \pm 0.061}}$ &                                  $0.046 \pm 0.074$ &                                 $0.153 \pm 0.002$ &                                 $0.019 \pm 0.008$ \\
                        & \textsc{zscore-random} &                                 $0.300 \pm 0.004$ &                                 $1.180 \pm 0.016$ &                                  $6.119 \pm 0.070$ &  $\cellcolor{green!25}{\mathbf{0.149 \pm 0.002}}$ &  $\cellcolor{green!25}{\mathbf{0.003 \pm 0.002}}$ \\
                        & \textsc{cdf-random} &                                 $0.343 \pm 0.017$ &                                 $1.179 \pm 0.016$ &                                  $6.126 \pm 0.066$ &  $\cellcolor{green!25}{\mathbf{0.149 \pm 0.002}}$ &  $\cellcolor{green!25}{\mathbf{0.003 \pm 0.002}}$ \\
                        & \textsc{zscore-linear} &                                 $0.300 \pm 0.004$ &                                 $1.180 \pm 0.016$ &                                 $-0.006 \pm 0.017$ &  $\cellcolor{green!25}{\mathbf{0.149 \pm 0.002}}$ &  $\cellcolor{green!25}{\mathbf{0.003 \pm 0.002}}$ \\
\midrule \texttt{wine} & \textsc{zscore-NAF} &                                 $0.830 \pm 0.021$ &                                 $3.794 \pm 0.323$ &                                  $1.376 \pm 0.066$ &  $\cellcolor{green!25}{\mathbf{0.427 \pm 0.012}}$ &  $\cellcolor{green!25}{\mathbf{0.013 \pm 0.007}}$ \\
                        & \textsc{cdf-linear} &                                 $1.031 \pm 0.258$ &                                 $3.256 \pm 0.135$ &                                  $1.513 \pm 0.046$ &  $\cellcolor{green!25}{\mathbf{0.427 \pm 0.012}}$ &  $\cellcolor{green!25}{\mathbf{0.013 \pm 0.007}}$ \\
                        & \textsc{cdf-NAF} &                                 $0.844 \pm 0.038$ &                                 $3.308 \pm 0.181$ &   $\cellcolor{green!25}{\mathbf{1.341 \pm 0.185}}$ &                                 $0.433 \pm 0.011$ &                                 $0.015 \pm 0.007$ \\
                        & \textsc{zscore-random} &                                 $0.791 \pm 0.020$ &                                 $3.256 \pm 0.136$ &                                 $12.452 \pm 0.122$ &  $\cellcolor{green!25}{\mathbf{0.427 \pm 0.012}}$ &  $\cellcolor{green!25}{\mathbf{0.013 \pm 0.007}}$ \\
                        & \textsc{cdf-random} &                                               N/A &                                 $3.256 \pm 0.136$ &                                 $12.436 \pm 0.101$ &                                               N/A &  $\cellcolor{green!25}{\mathbf{0.013 \pm 0.007}}$ \\
                        & \textsc{zscore-linear} &  $\cellcolor{green!25}{\mathbf{0.789 \pm 0.020}}$ &  $\cellcolor{green!25}{\mathbf{3.253 \pm 0.132}}$ &                                  $1.589 \pm 0.055$ &  $\cellcolor{green!25}{\mathbf{0.427 \pm 0.012}}$ &  $\cellcolor{green!25}{\mathbf{0.013 \pm 0.007}}$ \\
\midrule \texttt{yacht} & \textsc{zscore-NAF} &  $\cellcolor{green!25}{\mathbf{0.066 \pm 0.007}}$ &  $\cellcolor{green!25}{\mathbf{0.249 \pm 0.028}}$ &  $\cellcolor{green!25}{\mathbf{-1.484 \pm 0.180}}$ &  $\cellcolor{green!25}{\mathbf{0.042 \pm 0.007}}$ &  $\cellcolor{green!25}{\mathbf{0.015 \pm 0.011}}$ \\
                        & \textsc{cdf-linear} &  $\cellcolor{green!25}{\mathbf{0.066 \pm 0.007}}$ &                                 $0.266 \pm 0.040$ &                                 $-1.134 \pm 0.171$ &  $\cellcolor{green!25}{\mathbf{0.042 \pm 0.007}}$ &                                 $0.016 \pm 0.011$ \\
                        & \textsc{cdf-NAF} &  $\cellcolor{green!25}{\mathbf{0.066 \pm 0.007}}$ &                                 $0.260 \pm 0.037$ &                                 $-1.470 \pm 0.178$ &  $\cellcolor{green!25}{\mathbf{0.042 \pm 0.007}}$ &  $\cellcolor{green!25}{\mathbf{0.015 \pm 0.011}}$ \\
                        & \textsc{zscore-random} &                                 $0.098 \pm 0.006$ &                                 $0.266 \pm 0.039$ &                                 $12.958 \pm 0.316$ &  $\cellcolor{green!25}{\mathbf{0.042 \pm 0.007}}$ &                                 $0.016 \pm 0.011$ \\
                        & \textsc{cdf-random} &                                               N/A &                                 $0.267 \pm 0.041$ &                                 $12.996 \pm 0.323$ &                                               N/A &                                 $0.016 \pm 0.011$ \\
                        & \textsc{zscore-linear} &                                 $0.068 \pm 0.009$ &                                 $0.259 \pm 0.037$ &                                 $-1.163 \pm 0.176$ &  $\cellcolor{green!25}{\mathbf{0.042 \pm 0.007}}$ &                                 $0.016 \pm 0.011$ \\
\bottomrule
\end{tabular}}
    \caption{Experimental results for individual datasets.}    
    \label{tab:my_label}
\end{table}

\end{document}